%% file: main_mlr.tex
\documentclass{applemlr}
\usepackage{amsmath}
\usepackage{enumerate} 
\usepackage{amsfonts}
\usepackage{amsthm}
\usepackage{newtxtt}
\usepackage{cleveref}
\usepackage{diagbox} 
\usepackage{colortbl}
\usepackage{amssymb}
\usepackage{xspace}
\usepackage{wrapfig}
\usepackage{adjustbox}
\usepackage{tabularx}
\usepackage{booktabs}
\usepackage{mathtools}
\usepackage{tikz}
\usepackage{enumitem}
\usepackage{silence}
\usepackage{dsfont}
\usepackage[table]{xcolor}
\usepackage[dvipsnames]{xcolor}
\usepackage{multirow}
\usepackage{makecell}
\usepackage{xfakebold}
\input{math_commands}

\input{packages_and_commands_mlr.tex}
\usepackage{algorithm}
\usepackage{algorithmic}

\usepackage{tikz-cd}
\usetikzlibrary{calc}
\usepackage{pgfplots}
\usetikzlibrary{math}
\usepgfplotslibrary{groupplots}
\pgfplotsset{compat=1.3}
\usepackage{csvsimple}
\usepackage[disable]{todonotes}
\newcommand{\vnote}[1]{\todo[color=red!40, inline]{VF: #1}}
\definecolor{textgray}{HTML}{6E6E73}
\usetikzlibrary{positioning, calc}
\usetikzlibrary{decorations.pathmorphing}

\makeatletter
\patchcmd{\wrong@fontshape}{\@gobbletwo}{}{}{}
\makeatother
\numberwithin{equation}{section} 
\tcbuselibrary{minted}
\setminted[python]{
  linenos,
  breaklines,
  fontsize=\footnotesize,
  xleftmargin=2em
}

\makeatletter
\AtBeginDocument{
  \urlstyle{sf}
  
}
\makeatother

\definecolor{light}{RGB}{125, 125, 125}
\crefname{tcb@cnt@pbox}{code}{code}
\Crefname{tcb@cnt@pbox}{Code}{Code}
\crefname{assumption}{assumption}{assumption}
\Crefname{assumption}{Assumption}{Assumptions}
\Crefname{definition}{Definition}{Definitions}
\crefname{definition}{definition}{definitions}
\crefname{corollary}{corollary}{corollaries}
\Crefname{corollary}{Corollary}{Corollaries}
\crefname{proposition}{proposition}{propositions}
\Crefname{proposition}{Proposition}{Propositions}

\newtcolorbox[auto counter]{pbox}[2][]{
  colback=white,
  title=Code~\thetcbcounter: #2,
  #1,fonttitle=\sffamily,
  fontupper=\sffamily,
  arc=2pt,
  colframe=bgcolor,
  coltitle=fgcolor,
  colbacktitle=bgcolor,
  toptitle=0.25cm,
  bottomtitle=0.125cm
}

\makeatletter
\newcommand\applefootnote[1]{%
  \begingroup
  \renewcommand\thefootnote{}%
  \renewcommand\@makefntext[1]{\noindent##1}%
  \footnote{#1}%
  \addtocounter{footnote}{-1}%
  \endgroup
}
\makeatother

\definecolor{cverbbg}{gray}{0.90}

\usepackage{listings}
\title{Cram Less to Fit More: Training Data Pruning Improves Memorization of Facts}

\author{Jiayuan Ye$^\dagger$$^\ddagger$$^*$ \qquad Vitaly Feldman$^\ddagger$ \qquad Kunal Talwar$^\ddagger$}

\affiliation{$^\dagger$ National University of Singapore\qquad $^\ddagger$Apple}

\date{\sffamily\today}
\abstract{Large language models (LLMs) can struggle to memorize factual knowledge in their parameters, often leading to hallucinations and poor performance on knowledge-intensive tasks. In this paper, we formalize fact memorization from an information-theoretic perspective and study how training data distributions affect fact accuracy. We show that fact accuracy is suboptimal (below the capacity limit) whenever the amount of information contained in the training data facts exceeds model capacity. This is further exacerbated when the fact frequency distribution is skewed (e.g. a power law). 
    We propose data selection schemes based on the training loss alone that aim to limit the number of facts in the training data and flatten their frequency distribution. On semi-synthetic datasets containing high-entropy facts, our selection method effectively boosts fact accuracy to the capacity limit. When pretraining language models from scratch on an annotated Wikipedia corpus, our selection method enables a GPT2-Small model (110m parameters) to memorize 1.3X more entity facts compared to standard training, matching the performance of a 10X larger model (1.3B parameters) pretrained on the full dataset.  }

\begin{document}

\maketitle
\begingroup
\renewcommand\thefootnote{}\footnotetext{*~Research done while at Apple.}
\endgroup

\input{intro}

\input{related_works}

\input{formulations}

\input{conclusion}

\input{acks}

\bibliography{reference}
\bibliographystyle{plainnat}

\newpage

\appendix
\tableofcontents
\input{app_main}

\applefootnote{ \textcolor{textgray}{\sffamily Apple and the Apple logo are trademarks of Apple Inc., registered in the U.S. and other countries and regions.}}

\end{document}

%% file: math_commands.tex
\usepackage{amsmath,amsfonts,bm}

\def\eqref#1{equation~\ref{#1}}

\def\1{\bm{1}}

\DeclareMathAlphabet{\mathsfit}{\encodingdefault}{\sfdefault}{m}{sl}
\SetMathAlphabet{\mathsfit}{bold}{\encodingdefault}{\sfdefault}{bx}{n}

%% file: packages_and_commands_mlr.tex
\usepackage{graphicx}
\usepackage{tabularx}
\usepackage{multirow} 
\usepackage{xcolor}
\usepackage{geometry}
\usepackage{booktabs}
\usepackage{fullpage}

\usepackage[T1]{fontenc}
\usepackage[utf8]{inputenc}
\usepackage{url}

\usepackage{multirow}
\usepackage{booktabs}

\usepackage{xcolor}
\usepackage{adjustbox}
\usepackage{tabularx}
\usepackage{subcaption}

\usepackage{graphicx}

\usepackage{amsmath}
\usepackage{amsmath}
\usepackage{amssymb}
\usepackage{mathtools}
\usepackage{amsthm}
\usepackage{amsfonts}
\usepackage{bm}
\usepackage{enumitem}

\theoremstyle{plain}
\newtheorem{theorem}{Theorem}[section]
\newtheorem{proposition}[theorem]{Proposition}

\newtheorem{corollary}[theorem]{Corollary}

\theoremstyle{definition}
\newtheorem{definition}[theorem]{Definition}

\theoremstyle{remark}

\def\compileFigures{0}

\usepackage{tikz}
\if\compileFigures1
\usetikzlibrary{external}
\fi

\usepackage{tikz-cd}
\usetikzlibrary{calc}
\usepackage{pgfplots}
\usepackage{pgfplotstable}
\usepackage{subcaption}
\pgfplotsset{compat=1.17}
\usetikzlibrary{math}
\usepgfplotslibrary{groupplots}
\usepackage{csvsimple}
\usepgfplotslibrary{fillbetween}

\usepackage{bm}

\usepackage{wrapfig}

\newdimen\figrasterwd
\figrasterwd\textwidth

\usepackage{tikz}

\usepackage{caption}

%% file: intro.tex
\section{Introduction}

Machine learning models memorize training data by encoding information into their parameters. Such memorization, in many scenarios, is desirable behavior, e.g., for learning factual knowledge in the training data~\citep{roberts2020much,kadavath2022language,pagnoni2021understanding} or for generative retrieval~\citep{de2020autoregressive,tay2022transformer,pradeep2023does}.~\footnote{Memorization is also an undesirable thing in certain situations, due to overfitting or privacy concerns.} Moreover, when the training data distribution is long-tailed, which is often the case in practice~\citep{kandpal2023large,mallen2023not,zhu2014capturing,van2017devil}, the memorization of data that contain rare knowledge is known to be theoretically necessary for accurate learning and generalization~\citep{feldman2020does,brown2021memorization,feldman2025trade}.
However, it has been observed that current language models do not succeed in memorizing data in the tail of the distribution~\citep{kandpal2023large,mallen2023not}. Notably, even state-of-the-art frontier models achieve less than 50\% accuracy on challenging closed-book Q\&A~\footnote{Closed-book Q\&A refers to directly using a model to answer questions without external context or knowledge base.} benchmarks like SimpleQA~\citep{wei2024measuring} (at the time of its release). Scaling up model size~\citep{roberts2020much,kandpal2023large} is observed to boost fact accuracy after training, but only at a slow log-linear rate --- it was predicted in \citet[Figure 6]{kandpal2022deduplicating} that an exceedingly large model ($10^{20}$ parameters) would be needed for memorizing all the facts in Wikipedia to the level of human accuracy. Such scaling, if accurate, would indicate that high fact accuracy is practically impossible under any realistic model size. This raises two fundamental questions: (1) Is limited fact accuracy a theoretical inevitability, or does it arise from suboptimal training data distributions? (2) If the latter, can we design data selection schemes that approach the theoretical capacity limit? In this paper, we aim to understand and address these two questions, by defining, measuring, and ultimately boosting the memorization of facts in language models.

\input{test_fact_accuracy_vs_model_size}
\textbf{Theoretical Capacity Limits} We first investigate the theoretical capacity limit of fact memorization and fact accuracy in language models. To formulate this problem, motivated by \citet{feldman2020does,brown2021memorization,allen2023physics,morris2025much}, we define fact memorization from an information theoretic perspective, and prove a new connection (\Cref{thm:acc_based_mem_lower}) between fact accuracy  and the 2 bits/parameter memorization capacity limit established in prior works~\citep{allen2024physics,morris2025much,gu2025data}. Through this connection, we prove capacity limit for fact accuracy,  in terms of an upper bound for the maximal number of independent facts that a language model can answer accurately (\Cref{cor:fact_mem_upper}).  We then experimentally investigate whether language model reaches this fact accuracy capacity limit under different training data distributions, through a series of synthetic power law phonebook memorization benchmarks~\citep{jelassi2024mixture} with varying numbers of facts and different fact frequency distributions.~\footnote{Note that the synthetic setup is necessary for precise measurements of fact entropy in deriving the fact accuracy capacity limit, as the amount of fact information in real-world datasets (such as wikipedia articles) is challenging to define, let alone to measure. We refer to~\citet[Section 1]{allen2024physics} for more discussions.} We observe that fact accuracy is suboptimal (below the capacity limit) whenever the amount of information contained in the training data facts exceeds model capacity (\Cref{fig:capacity}). This is further exacerbated when the fact frequency distribution is skewed (e.g. a power law) as we show in Figures \ref{fig:power_gap} and \ref{fig:detailed_gap_power}. To rule out insufficient training steps (i.e., rare facts not being seen frequently enough), we also perform ablation experiments to validate that (1) a 10x larger model undergoing the same training procedure can perfectly answer all the facts; and (2) training with 8x more steps still yields similarly low fact accuracy. This shows that there is a fact accuracy gap between small and large models that gets exacerbated under non-uniform fact frequency distribution in the training dataset. In summary, the only setting where we observe language model to reach the fact accuracy capacity limit is for training on \textit{proportionally} many \textit{uniformly} distributed facts.

\textbf{Data Selection for Improved Memorization of Facts} We then propose training data selection schemes to boost fact accuracy: limiting the number of facts in the training data to avoid exceeding model capacity, and removing excessive repetitions of frequent facts to get closer to uniform fact distribution. A key challenge in real-world datasets is that fact boundaries and frequencies are unknown. To circumvent this, we utilize training loss as a proxy for fact frequency in data selection. We show in \Cref{sec:optimal_selection} that our loss-based data selection methods (1) effectively boost fact accuracy to the capacity limit for training on semi-synthetic facts (including pretraining on synthetic phonebooks and LoRA finetuning on high-entropy title-author mapping facts in arXiv papers); (2) increase the accuracy for entity-facts by 1.3X for pretraining GPT2-Small model (110m parameters) from scratch on the Wikipedia corpus, matching the fact accuracy of a 10X larger model (1.3B parameters) pretrained on the unselected full dataset. Our results show that there is significant room for improving fact accuracy, especially for small models, via training data selection.

\textbf{Summary of Results} Our main contributions are (1) we define fact memorization and use it to establish fact accuracy capacity limit in \Cref{sec:maximal_fact_memorization}; (2) we experimentally show in  \Cref{sec:cap_speed_mem_standard_training} that fact accuracy of standard training is suboptimal when the full training dataset contains too many or power-law distributed facts; (3) we propose loss-based data selection methods in \Cref{sec:optimal_selection} and  show how they boost fact accuracy under various semi-synthetic and real-world settings; (4) we additionally perform ablation experiments about the design choices and computation cost of our selection algorithms (deferred to \Cref{sec:ablation_and_discussions}).

%% file: test_fact_accuracy_vs_model_size.tex
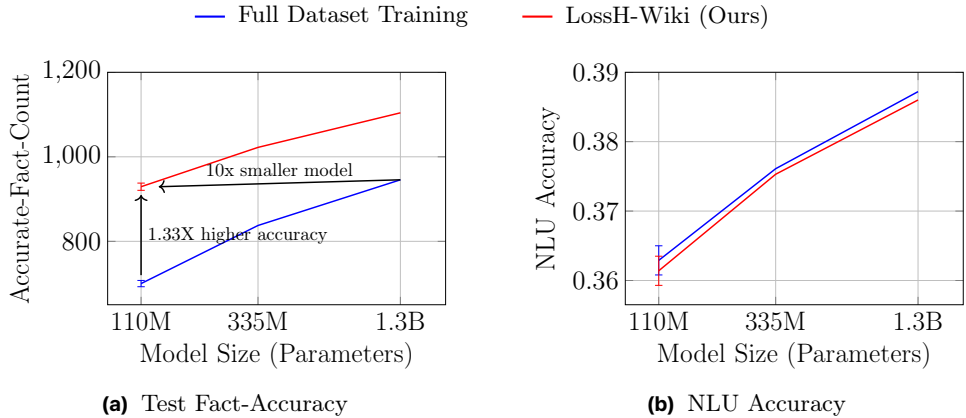
\begin{figure}
\centering
\begin{tikzpicture}[baseline=-0.5ex]
\draw[blue, thick] (0,0) -- (0.4,0);
\node[right] at (0.45,0) {\small Full Dataset Training};
\end{tikzpicture}
\hspace{1em}
\begin{tikzpicture}[baseline=-0.5ex]
\draw[red, thick] (0,0) -- (0.4,0);
\node[right] at (0.45,0) {\small LossH-Wiki (Ours)};
\end{tikzpicture}

\vspace{0.2cm}

\begin{subfigure}[b]{0.42\textwidth}
\centering
\scalebox{0.7}{\begin{tikzpicture}
\begin{axis}[
    width=8cm,
    height=6cm,
    xlabel={Model Size (Parameters)},
    ylabel={Accurate-Fact-Count},
    xmode=log,
    log basis x={10},
    xmin=80,
    xmax=2000,
    ymin=650,
    ymax=1200,
    xtick={110, 335, 1300},
    xticklabels={110M, 335M, 1.3B},
    grid=major,
    label style={font=\Large},
    tick label style={font=\Large},
    mark size=2pt,
]

\addplot[
    color=blue,
    no markers,
    thick,
] coordinates {
    (110, 700.05)
    (335, 837.53)
    (1300, 945.58)
};

\addplot[
    color=blue,
    no markers,
    error bars/.cd, y dir=both, y explicit,
] coordinates {
    (110, 700.05) +- (0, 7.42)
};

\addplot[
    color=red,
    no markers,
    thick,
] coordinates {
    (110, 929.26)
    (335, 1022.29)
    (1300, 1104.08)
};

\addplot[
    color=red,
    no markers,
    error bars/.cd, y dir=both, y explicit,
] coordinates {
    (110, 929.26) +- (0, 8.55)
};

\draw[->, thick, black] (axis cs:110, 720) -- (axis cs:110, 910)
    node[midway, right, font=\normalsize] {1.33X higher accuracy};

\draw[->, thick, black] (axis cs:1300, 945.6) -- (axis cs:130, 929.2)
    node[midway, above, font=\normalsize] {10x smaller model};

\end{axis}
\end{tikzpicture}}
\caption{Test Fact-Accuracy}
\label{fig:fact_accuracy}
\end{subfigure}%
\begin{subfigure}[b]{0.42\textwidth}
\centering
\scalebox{0.7}{\begin{tikzpicture}
\begin{axis}[
    width=8cm,
    height=6cm,
    xlabel={Model Size (Parameters)},
    ylabel={NLU Accuracy},
    xmode=log,
    log basis x={10},
    xmin=80,
    xmax=2000,
    xtick={110, 335, 1300},
    xticklabels={110M, 335M, 1.3B},
    grid=major,
    label style={font=\Large},
    tick label style={font=\Large},
    mark size=2pt,
]

\addplot[
    color=blue,
    no markers,
    thick,
] coordinates {
    (110, 0.3629)
    (335, 0.3761)
    (1300, 0.3872)
};

\addplot[
    color=blue,
    no markers,
    error bars/.cd, y dir=both, y explicit,
] coordinates {
    (110, 0.3629) +- (0, 0.0021)
};

\addplot[
    color=red,
    no markers,
    thick,
] coordinates {
    (110, 0.3614)
    (335, 0.3753)
    (1300, 0.3860)
};

\addplot[
    color=red,
    no markers,
    error bars/.cd, y dir=both, y explicit,
] coordinates {
    (110, 0.3614) +- (0, 0.0021)
};

\end{axis}
\end{tikzpicture}}
\caption{NLU Accuracy}
\label{fig:downstream_accuracy}
\end{subfigure}
\caption{Test Fact-Accuracy (in Accurate-Fact-Count) and NLU-Accuracy (natural language understanding) vs. Model Size for pretraining on annotated Wikipedia Corpus (around 3B tokens) for 66k steps with batch-size 320 and context length 1024 (around 8 epochs). Our LossH-Wiki selection \Cref{alg:cramless_wiki} significantly improves fact-accuracy without harming general natural language understanding. We show results for selection ratio $\alpha=0.2$ (110m) and $\alpha=0.3$ (335m, 1.3B) tuned to maximize the train fact-accuracy. For the small model with 110m parameters, we repeat 3 runs and report mean and standard deviation (shown in error bars). See performance under other selection ratios in \Cref{fig:lmlmwiki_loss}.}
\label{fig:accuracy_vs_model_size}
\end{figure}

%% file: related_works.tex
\section{Related work}
\label{sec:related_works}
\paragraph{Memorization Definition} There is a long line of literature on how to meaningfully define memorization for general machine learning algorithms, i.e., the information that a trained model contains about its training dataset. In the natural language modeling literature, such memorization is often intuitively defined by the probability of generating the complete training data given its prefix tokens~\citep{carlini2019secret,carlini2022quantifying,tirumala2022memorization,biderman2023emergent,biderman2023pythia}. However, such memorization definitions are often criticized for their overly small coverage, i.e., only capturing the model's prediction behavior on the exact prefix tokens of training data. To address this limitation, recent works~\citep{schwarzschild2024rethinking,morris2025much} relax the definition and define memorization by the shortest context required to generate a training data example, relating it to data compression and Kolmogorov complexity~\citep{kolmogorov1965three}. However, such memorization definitions still suffer from a lack of interpretability due to the fact that a language model can essentially predict any target string, given sufficient manipulations to the context~\citep{geiping2024coercing}. Consequently, such measurements are typically only considered meaningful when performed on training data that contains unique information, e.g., personally identifiable information~\citep{carlini2019secret,carlini2023extracting,lukas2023analyzing,biderman2023pythia}, or when they are compared to measurements on leave-one-out models~\citep{feldman2020does,feldman2020neural,zhang2023counterfactual,ye2023leave}.

Our work focuses on measuring the relevant information that the model contains via mutual information between the model and facts about the data distribution. Using mutual information to measure ``useful" information that a model contains about the data distribution is standard in ML and statistics~\citep{lindley1956measure,blumer1987occam,tishby2000information}. Recent work of \citet{brown2021memorization} also proposes using mutual information between the dataset and the model that is conditioned on the data distribution to measure memorization of irrelevant information. In particular, this information excludes any information about the data distribution. This notion has been investigated theoretically in \citep{attias2024information,feldman2025trade} and empirically in \citep{morris2025much}.

A key benefit of memorization definitions based on mutual information is their intuitive interpretations in terms of bits of information. \citet{allen2023physics,allen2024physics} derive information-theoretical connections between a loss-based memorization definition and  the bit precision and number of parameters of language models. \citet{morris2025much,gu2025data} prove similar connections between the mutual-information-based memorization definition and the bit precision and number of parameters of language models. Our work builds on these works, but focuses on fact memorization and its relationship to the memorization capacity limit.

\paragraph{Memorization Capacity Limit Analysis} Our work theoretically analyzes and experimentally validates the fact accuracy capacity limit of language models via controlled experiments on synthetic fact memorization benchmarks. We directly build upon a recent line of works that analyzes and measures memorization capacity for language model training on randomly constructed synthetic datasets~\citep{allen2024physics,zucchet2025language,gu2025data,morris2025much,jelassi2024mixture}.  To the best of our knowledge, all of the aforementioned works focus on relating memorization to training loss, while we prove new connections between  (fact) memorization and the stricter metric of fact accuracy that is more suitable for measuring performance of fact-learning. Additionally, prior works focus on random facts with uniform frequency distribution~\citep{allen2024physics,morris2025much,jelassi2024mixture}, and only perform passive analysis of memorization speed under power-law distributed facts~\citep{gu2025data,zucchet2025language}.  By contrast, we propose to \textit{actively} manipulate the training data distribution via loss-based data selection (that limits the number of facts in the training data and flattens their frequency distribution). Notably, we observe that data selection can boost fact accuracy to the (newly derived) capacity limit for training language models across a range of synthetic data distributions. Moreover, our data selection schemes apply to pretraining on general Wikipedia corpus, and effectively boosts fact accuracy and knowledge-intensive downstream task performances without harming natural language understanding.  We also remark that although \citet{gu2025data,jelassi2024mixture} propose other data paraphrasing/augmentation/mixing techniques to improve fact accuracy, their techniques are heavily dependent on the restricted set of fact formats in their synthetic setups. By contrast, our proposed data selection technique is solely based on loss and annotated fact tokens, thus more generally applicable to real-world datasets such as arXiv papers and Wikipedia datasets, as shown in our experiments.

\paragraph{Parametric Memory of Language Models} Parametric memory~\citep{roberts2020much,fevry2020entities,de2020autoregressive,yu2022generate,kandpal2023large,meng2022locating} refers to the practice of using a language model to answer questions in a closed-book manner, without access to any external information. This requires a language model to encode knowledge in its parameters, and has been a long-standing topic in understanding the knowledge capabilities of language models~\citep{petroni2019language,jiang2020can,chowdhery2023palm,shah2025beyond}. Such parametric memory, is often evaluated by closed-book  Q\&A performance on knowledge-based tasks for short-form generations~\citep{wei2024measuring,kwiatkowski2019natural,berant2013semantic,joshi2017triviaqa,mallen2023not,srivastava2023beyond}, or by factuality performance for long-form generations~\citep{min2023factscore,li2023halueval}. It has been observed~\citep{kandpal2023large,mallen2023not} that the parametric memory of a language model is limited when trained on long-tailed data distributions, especially in its  knowledge of rare facts. However, to the best of our knowledge, most of the prior investigations are experimentally conducted under fixed training algorithms and real-world datasets, where there are many confounding factors that hinder understanding whether the limited parametric memory is theoretically inevitable, or is only a consequence of suboptimal training algorithms or suboptimal training dataset. By contrast, we theoretically and experimentally analyze the capacity limit of fact accuracy under the more controlled settings of synthetic long-tailed fact datasets. Our controlled investigations show that limited fact accuracy is largely due to suboptimal training data selection. Importantly, we show that simple loss-based data selection schemes can boost fact accuracy to the capacity limit even when standard training on the full data distribution is highly suboptimal and only achieves close-to-zero fact accuracy.

\paragraph{Contextual Memory of Language Models} Although our paper focuses on boosting the parametric memory of a language model, we remark that an orthogonal line of effort addressing the limited parametric memory uses contextual memory~\citep{berant2013semantic,chen2017reading,radford2019language}, where the information is fed to the model as input context. To this end, contextual memory predominantly requires a language model's reading comprehension capabilities, thus largely alleviating the burden of memorizing all knowledge in model parameters. The most prevalent examples of contextual memory include retrieval augmented generation (RAG)~\citep{khandelwal2019generalization,karpukhin2020dense,izacard2022few,min2023nonparametric} and tool-calling~\citep{schick2023toolformer,jin2025search,parisi2022talm}. The performance and efficiency of using contextual memory crucially relies on the quality and quantity of external information sources, besides the language model's knowledge as well as reading comprehension capabilities. Due to the numerous confounding factors, it is often tricky to separate out different factors that affect contextual memory. In this paper, we solely focus on boosting language model's parametric memory, which is an orthogonal direction to improving usability of contextual memory -- it is often observed~\citep{jin2025search,zhang2025reinforcement} that the performance and efficiency of knowledge-intensive Q\&A often improve as the parametric memory of the base model becomes stronger. We show through experiments that our insights generalize to real-world dataset, e.g., boosting the fact accuracy and accuracy of knowledge downstream tasks for pretraining from scratch on Wikipedia corpus.

\paragraph{Data Selection for Language Model Training}  A key ingredient of our data selection technique is selecting and flattening the head of the distribution. Intuitively, our head selection step has a similar effect to a long line of works in the literature for selecting data that are the most aligned to end tasks~\citep{xie2023data,fan2023doge,xia2024less,grangier2024task,wang2024greats}, and our flattening operation is intuitively similar to data de-duplication~\citep{lee2022deduplicating,kandpal2022deduplicating} and diversification~\citep{tirumala2023d4,jung2025prismatic,wang2024diversity,sachdeva2024train,liu2023makes}. However, the aforementioned works largely neglect the question of the right amount of data to subselect by treating it as a hyperparameter tuning problem in experiments. By contrast, one key contribution of our paper is to provide an understanding on the right amount of data to select from the model capacity limit perspective. In other words, the optimal selection for a small model is intrinsically different from that of a large model, which is a perspective missing in the prior data selection literature.

In terms of selection score, we use loss as a proxy to approximate the weights and bits of information in training data. The  signal of loss itself is also widely used in the data selection literature to approximate sample alignment, learnability, difficulty, and diversity~\citep{lin2024not,mindermann2022prioritized,yu2024mates,engstrom2024dsdm,li2024quantity,sanyal2025upweighting}. Different from the prior data selection objectives, our data selection algorithms maximizes fact accuracy within the fact memorization capacity limit. Importantly, for training on synthetic power-law distributed random phonebook facts, we are able to experimentally show that our data selection scheme is near-optimal for the fact accuracy objective, effectively boosting fact accuracy to close to the (2 bits/parameter) capacity limit.

Lastly, we remark that our selection strategy of prioritizing low loss facts is effectively the opposite of the seminal Rho-1~\citep{lin2024not} and Rho-Loss~\citep{mindermann2022prioritized}  algorithms (that select data with high excess loss relative to a reference model). As our selection is experimentally near-optimal for memorizing independent facts, this suggests that the benefits of Rho-1~\citep{lin2024not} and Rho-Loss~\citep{mindermann2022prioritized} are likely not due to better fact memorization. Indeed, these works typically evaluate more complex learning tasks, which not only require fact memorization but also require generalization and even reasoning, making selecting low-loss samples not necessarily ideal.
We leave it as an interesting open question as to which data selection schemes achieve the ideal trade-offs between fact accuracy and other model capabilities.

%% file: formulations.tex
\section{What Fact Accuracy is Possible?}
\label{sec:maximal_fact_memorization}

In this section, we study what level of fact accuracy is theoretically possible and experimentally reachable for a language model.  We start from the necessary definitions. 

\paragraph{Facts as Deterministic Mappings} Let $\theta \in \Theta$ be the parameter that determines the data distribution, i.e., the description of the world from which we obtain the data. Intuitively, the description of a world $\theta$ includes numerous facts about the world. Each specific fact can be thought of as the answer to some question $Q$, such as ``When was Albert Einstein born?'' with the answer being determined by the data distribution. Therefore we define a fact about a world using a pair $(Q,A)$, where $Q$ is a fixed question (or prompt) that does not depend on the distribution and $A \colon \Theta \to Y$ is the mapping from world $\theta$ to the value of the answer to $Q$ in $\theta$. Here $Y$ is the range of possible values of the answer and could, for example, be all dates in some range. 

\begin{definition}[Facts about Data Distribution]
Let the training data distribution be parameterized by $\theta \in \Theta$. We say that question-answer pairs $(Q_i, A_i)_{i=1}^N$ are facts about the data distribution, if each $Q_i$ for $i\in[N]$ is a question string and each $A_i\colon \Theta \to Y_i$ is a deterministic answer function that provides the answer to $Q_i$ in $\theta$. \label[definition]{def:facts}
\end{definition}
Essentially, we treat `facts' as deterministic mappings inherent to the world state $\theta$. A learning algorithm has uncertainty about the value of $\theta$ and thus also about the values of facts in $\theta$.
We will model uncertainty about the world using a prior (meta-)distribution $\Psi$ over $\Theta$.
\begin{definition}[Dataset distribution] A dataset consists of $n$ i.i.d. samples from a data distribution $P_\theta$ parameterized by $\theta$, where the data distribution parameter $\theta$ is drawn from a meta prior distribution denoted by $\Psi$. \label[definition]{def:dataset}
\end{definition}

\begin{definition}[Learning Algorithm]
    A learning algorithm $\mathcal{A}$ takes as input a dataset $D$ and outputs a trained model  $\mathcal{A}(D)\in \mathcal{W}$ in a discrete model space $\mathcal{W}$.
    \label[definition]{def:learning_alg}
\end{definition}

We treat the task of memorizing a fact as binary: the fact is either answered correctly or incorrectly.
In this paper, we aim to maximize fact accuracy defined as follows.

\begin{definition}[Fact accuracy] \label[definition]{def:fact_acc}
    Let the training data distribution be  $\mathcal{P}_\theta$ parameterized by $\theta$.  Let $(Q_i, A_i)_{i=1}^N$ be facts underlying the training data distribution as defined per \Cref{def:facts}. We define fact accuracy of a learning algorithm given $n$ data samples on facts $(Q_i, A_i)_{i=1}^N$, as the expected fraction of correctly predicted fact answers by the trained model as follows.
    \begin{align}
        \underset{(Q_i, A_i)_{i=1}^N}{\text{Acc}} \left(\mathcal{A}; \theta, n\right) = &   \frac{\sum\limits_{i=1}\limits^N\underset{{D\sim\mathcal{P}_\theta^n, \mathcal{A}}}{\Pr}\left[f(\mathcal{A}(D); Q_i)=A_i(\theta)\right]}{N} \label{eqn:fact_acc}
    \end{align}
    where $f(\mathcal{A}(D);Q_i)$ denotes the prediction of trained model $\mathcal{A}(D)$ on fact question $Q_i$. We also refer to the numerator of \eqref{eqn:fact_acc} as accurate fact count, denoted by $\underset{(Q_i, A_i)_{i=1}^N}{\text{Acc-Cnt}} \left(\mathcal{A}; \theta, n\right)$.
\end{definition}

The prediction function $f$ may be task-specific, e.g., for multiple-choice questions, the answer is typically selected from a finite set so as to maximize its log-probability-score on the model given context $Q_i$; for free-form question answering, the prediction function is often top-k decoding of the language model given context $Q_i$ combined with text filtering (e.g., removing trailing white space). For simplicity, unless otherwise stated, we consider $f$ to be vanilla stochastic-decoding given context $Q_i$, that is, $\Pr\left[f(\mathcal{A}(D); Q_i)=a\right] = e^{-\ell(a; \mathcal{A}(D), Q_i)}$ where $\ell(a; \mathcal{A}(D), Q_i)$ is the sum of per-token cross-entropy loss of the trained model $\mathcal{A}(D)$ on answer $a$ given context $Q_i$.

\paragraph{Information-Theoretic Fact Memorization}
To study what fact accuracy is (im)possible, we now translate the well-established information capacity limit (i.e., the number of bits that a language model can memorize) to fact accuracy. The information that a learning algorithm extracts from the input dataset can be partitioned into two parts: the relevant information about the data distribution, and the remaining (or excess) memorization that is specific to the input dataset \citep{brown2021memorization,feldman2025trade,morris2025much}. We are only interested in the memorization of relevant information about the facts underlying the training data, thus we focus on a new notion of fact memorization defined as follows. (See \Cref{app:additional_mem_def_discussion} for a more detailed comparison to prior memorization definitions.) 

\begin{definition}[Fact Memorization] \label{def:fact_mem}
    Let  $\theta\sim\Psi$ be drawn from a meta distribution $\Psi$. Let $(Q_i, A_i)_{i=1}^N$ be facts about the training data distribution as defined in \Cref{def:facts}. We define fact memorization of a learning algorithm $\mathcal{A}$ about facts $(Q_i, A_i)_{i=1}^N$ as the mutual information between the fact values in $\theta$ and the trained model $\mathcal{A}(D)$, as follows.
    \begin{align}
        \underset{(Q_i, A_i)_{i=1}^N}{\text{Mem}}\left(\mathcal{A}; \Psi, n\right) = &I\left((A_1(\theta),\cdots,A_N(\theta)); \mathcal{A}(D)\right)  \text{ where }\theta\sim\Psi, D\sim\mathcal{P}_\theta^n. \label{eqn:def_mem}
    \end{align}
\end{definition}
It is straightforward to prove (\Cref{prop:axiom_mem}) that fact memorization is upper bounded by the size of the output space for the learning algorithm in the form of $ \underset{(Q_i, A_i)_{i=1}^N}{\text{Mem}}\left(\mathcal{A}; \Psi, n\right)\leq \ln|\mathcal{W}|$, analogous to similar bounds in prior works \citep[Theorem 3.2]{allen2024physics}, \citep{gu2025data} for other notions of memorization. For most practical learning algorithms, the output space $\mathcal{W}$ is indeed discrete due to the finite precision of model parameters, i.e.,   $\ln |\mathcal{W}| \propto \text{bit precision} \times \text{number of parameters}$, where bit precision typically equals $16$ or $32$ under using 16-bit or 32-bit floats. Essentially, we view  \Cref{prop:axiom_mem} as an axiom that any reasonable information-theoretic memorization definition should satisfy.

\subsection{Theoretical Capacity Limit for Fact Accuracy }  
\label{ssec:relate_fact_acc_to_mem}

We now use the $\ln|\mathcal{W}|$ capacity limit of fact memorization to understand the capacity limit of fact accuracy. The key challenge is to tightly relate fact memorization to fact accuracy. To address this challenge, we use Fano's inequality to prove a lower bound for fact memorization in terms of per-fact accuracy as follows. (Proofs are deferred to \Cref{app:proof_acc_based_mem_lower}.)

\begin{theorem}[Lower Bounding Fact Memorization by Per-fact Accuracy]\label{thm:acc_based_mem_lower}
    For any facts $(Q_i, A_i)_{i=1}^N$, any meta prior~$\Psi$ and any dataset size $n$,  we have 
    \begin{align}
    \underset{(Q_i, A_i)_{i=1}^N}{\text{Mem}}\left(\mathcal{A}; \Psi, n\right)
    \geq    
    {H}\left[(A_1(\theta), \cdots, A_N(\theta))\right]   
    - \sum_{i=1}^N\left(
     \Pr\left[I_i=0\right] 
    \cdot
    {H}\left[A_i(\theta)\mid I_i=0\right] + {H}\left[I_i\right]\right)
    \label{eqn:def_mem_pred}
    \end{align}
    where $I_i = \mathbf{1}_{f(\mathcal{A}(D);Q_i)= A_i(\theta)}$ is the accuracy indicator of the trained model on fact $i$, $f(\mathcal{A}(D);Q_i)$ denotes the prediction by the trained model  $\mathcal{A}(D)$ on question $Q_i$, and probability is over the randomness of $\mathcal{A}$, $\theta\sim\Psi$, and $D\sim\mathcal{P}_\theta^n$.  
\end{theorem}

We note that in this definition, the entropy $H$ of a random variable refers to the entropy of the distribution of that random variable.
\input{figures/fig_power_gap}
\Cref{thm:acc_based_mem_lower} relates overall fact accuracy to capacity, and generally applies to facts with arbitrary entropy distributions. To the best of our knowledge,  prior works~\citep{allen2024physics,gu2025data,morris2025much,pan2025understanding} focus on relating memorization to loss,  and are thus insufficient for understanding the stricter fact accuracy metric. (We refer to \Cref{def:loss_mem_prior_works} for more details on what prior loss-based memorization lower bounds translate to under our fact memorization definition.) 

In the literature, for simplicity, fact-learning experiments typically consider synthetic independently distributed random facts with fixed entropy, such as learning phone numbers~\citep{jelassi2024mixture},  biographies~\citep{allen2023physics,allen2024physics,gu2025data,zucchet2025language}, and even random strings of fixed length~\citep{carlini2019secret,morris2025much}. In this case,  \Cref{thm:acc_based_mem_lower} yields a clean capacity upper bound on the number of accurately answerable facts, as follows. (Proofs are deferred to \Cref{appcor:fact_mem_upper}.)

\begin{corollary}[Accurate Fact Count Capacity Limit on Fixed-Entropy Random Facts] As a special case of \Cref{thm:acc_based_mem_lower}, if each answer follows uniform distribution over a discrete answer domain $\mathcal{M}_i$ with $\ln|\mathcal{M}_i|=b$, and if $A_1(\theta), A_2(\theta), \cdots, A_N(\theta)$ are independent over the meta prior $\theta\sim\Psi$,    then the accurate fact count of any learning algorithm $\mathcal{A}$ satisfies
    \begin{align}
        \underset{\theta\sim\Psi}{\mathbb{E}}\left[\underset{(Q_i, A_i)_{i=1}^N}{\text{Acc-Cnt}}\left(\mathcal{A}; \theta, n\right)\right] \leq  \frac{\ln|\mathcal{W}| + N\cdot \ln 2}{b} \label{eqn:fact_mem_upper}
    \end{align}
    \label[corollary]{cor:fact_mem_upper}
\end{corollary}
\Cref{cor:fact_mem_upper} establishes that a language model in  discrete space $\mathcal{W}$ can accurately answer at most $ O\left(\frac{\ln|\mathcal{W}| }{b} \right)$ independent facts each with entropy $b$, if $\ln |W|\geq \Omega\left(N \cdot \ln 2\right)$ (i.e., if the model has sufficient capacity to store at least one bit of information per fact, preventing the binary entropy terms $H[I_i]$ from dominating the bound in \cref{thm:acc_based_mem_lower}). If a training algorithm precisely knows the entropy $b$ and the data records that correspond to each fact, then it can trivially achieve this $\frac{\ln|\mathcal{W}|}{b}$ capacity limit, by training on a subset of  $\frac{\ln|\mathcal{W}|}{b}$ facts until convergence, i.e., close-to-zero loss. By contrast, an algorithm that blindly trains on the full dataset intuitively only memorizes at most $\frac{\ln|\mathcal{W}|}{N}$ bits for each fact, which is insufficient to answer any single fact accurately if $\frac{\ln|\mathcal{W}|}{N}\ll b$. As we will see in our experiments, standard language model training indeed suffers from suboptimal fact accuracy when the amount of information in the training dataset exceeds model capacity.

\subsection{Experimental Evidence of Suboptimal Fact Accuracy in Standard Training}
\label{sec:cap_speed_mem_standard_training}

In this section, we experimentally investigate whether a language model reaches its fact accuracy capacity limit (\Cref{cor:fact_mem_upper}) under various training data distributions. To avoid the confounding prior knowledge in off-the-shelf language models, we focus on pretraining from scratch GPT2-style transformer models of different scales. (See \Cref{ssec:sufficient_training_setup} for the details of pretraining setups). Similar results hold for LoRA finetuning of pretrained Llama3.1-1B model, see  \Cref{app:suboptimal_fact_mem_LoRA} for more details.) To have full control of the number of facts and fact frequency distribution in the training dataset, we follow prior works~\citep{jelassi2024mixture,allen2024physics,gu2025data,zucchet2025language} and train on synthetically generated random phonebook facts, constructed as follows.

\paragraph{Simulating Long-Tail Facts via Synthetic Phonebooks} Each fact is a (name, phone number) tuple of the format <bos><6 alphabet tokens>|<22 digit tokens><eos>, where the name contains six randomly drawn alphabetical tokens from a to z, and the phone number contains 22 randomly sampled digits from 0 to 9. For tokenization,  our vocabulary is small, only containing a total of $39$ tokens including:  digit tokens 0-9, alphabet tokens a-z, separator token $|$, and special beginning of sentence <bos> and end of sentence <eos> tokens. 

To generate different number of phonebook facts $(Q_i,A_i(\theta))_{i=1}^N$ (where $Q_i$ and $A_i(\theta)$ denote the name and phone number for the $i$-th fact respectively), we sample $Q_1,\cdots,Q_N$ uniformly \textit{without} replacement from the name space, and sample the answers $A_1(\theta),\cdots,A_N(\theta)$ independently  from uniform distribution over all possible phone numbers. (This ensures that each name only maps to one phone number, and that $A_1(\theta),\cdots,A_N(\theta)$ are independent.) To vary the non-uniformity of fact frequency for $(Q_i,A_i)_{i=1}^N$ in the training data distribution,  we artificially sample each fact according to power law, i.e., $\Pr[x_j=(Q_i,A_i)]\propto 1/i^\beta, i\in[N]$, for $\beta=0, 0.5, 1.0$.

\paragraph{Fact Accuracy is Suboptimal When Data Exceeds Capacity}
 We first show in \Cref{fig:fact_acc_gap_uniform} (blue curve) that for training on uniformly random phonebook facts, the accurate fact count drops to zero whenever the number of facts in the training dataset is overly large. This is despite the training being sufficiently long for fact memorization to reach the 2 bits/parameter capacity limit, as shown in the ablation experiments in \Cref{ssec:additional_fact_memorization_phonebook} (\Cref{fig:capacity}). Intuitively, this means every fact is memorized to some extent, but no facts are fully memorized, i.e., the allocation of fact memorization capacity is suboptimal for maximizing fact accuracy.

\paragraph{Skewed Fact Distributions Widen the Capacity Gap} We further pretrain on power-law distributed phonebook facts, to understand the effect of non-uniformly distributed facts (which is often the case in real-world datasets) in the training dataset to fact memorization.  In \Cref{fig:power_gap} (blue curves), we observe that the highest achievable fact accuracy consistently drops under increasing power law exponent. This is despite the fact that the training is already sufficiently long such that an 8x longer training run yields negligible improvement, and that a 10X larger model trained on the same stream of data can answer all facts perfectly (as shown in ablation experiments in \Cref{ssec:additional_fact_memorization_phonebook}, \Cref{fig:detailed_gap_power}).

\section{Boosting Fact Accuracy via Data Selection}

\label{sec:optimal_selection}
How can we get a model to reach the maximal fact accuracy under model capacity constraints? In \Cref{sec:cap_speed_mem_standard_training}, we have observed that fact accuracy is suboptimal when the amount of information in the training data facts exceeds model capacity, especially under skewed fact frequency distributions. Notably, the only settings where fact accuracy reaches the theoretical capacity limit are for training on uniformly distributed facts whose number matches the model size. This motivates us to select training data to limit the number of facts and to flatten the fact frequency distribution (to be more uniform).

\paragraph{Our Loss-based Data Selection Algorithms} 
A key challenge is distinguishing rare from redundant facts without knowing fact frequencies, which is often the case for real-world datasets. To deal with this challenge, our key insight is that training loss serves as a proxy for fact frequency and entropy: low loss indicates a fact has been seen many times (high frequency) or is easy (low entropy), while high loss indicates rarity or high fact entropy. We propose two variants of loss-based selection in \Cref{alg:cramless}:
\begin{itemize}[noitemsep,topsep=0pt,leftmargin=*]
    \item \textbf{LossH (Head):} Select only low-loss samples, limiting training to memorizable facts.
    \item \textbf{LossHF (Head-Flattened):} Additionally downsample very low-loss samples to prevent excessive repetitions.
\end{itemize}
Both variants include the head selection step, which is \textit{necessary} for adapting the training data to the capacity of the underlying model. Meanwhile, the flattening step is \textit{optional}, and is only beneficial for further speeding up convergence when the training data is \textit{highly non-uniform} and \textit{already within model capacity}. (See \cref{sec:ablation_oracle} for ablation experiments on the isolated effects of head selection versus flattening.) Besides the central insights of capacity-aware head selection and flattening, our selection \cref{alg:cramless} also has three crucial design choices that jointly improve fact memorization.
\begin{enumerate}
    \item Fact-level selection: we compute per-fact loss and select at the unit of fact. This fact-level selection is necessary for boosting fact memorization, while standard (token-level) selection destroys fact boundaries and does not improve fact memorization (see ablation discussions in \cref{ssec:loss_selection_score}).
    \item Online loss threshold for fact selection: for each incoming batch of data, we recompute the threshold as the percentile of per-fact loss of the current model across the batch. As shown in \cref{fig:cali_speed}, this online threshold is necessary for improving the correlation between loss and the ground-truth fact frequency.
    \item To better calibrate loss to the underlying fact difficulty, we use \textit{sum of loss} over all tokens in each fact, rather than average of loss over tokens in the fact. This is because in ablation experiments (\cref{fig:cal_seq_len}), we observe that sum-of-loss yields stronger rank correlation to the fact weight-to-bits ratio. Intuitively, sum of loss indeed more closely corresponds to the semantic meaning of fact entropy.  
\end{enumerate}

We next experimentally validate that our data selection  \Cref{alg:cramless} effectively boosts fact accuracy across a broad range of synthetic and real-world settings. We defer the ablation experiments on design choices and computation cost of our selection algorithm in \cref{sec:ablation_and_discussions}.

\begin{algorithm}[t!]
	\caption{LossH / LossHF Selective Training (One Step)}
	\label{alg:cramless}
	\begin{algorithmic}
		\STATE {\bfseries Input:} data selection ratio $\alpha$, current model $\theta_t$ at iteration $t$, target batch size $b$
        \STATE {\bfseries Initialize Selected Batch:} $B_t \gets \emptyset$
        \STATE {\bfseries Data Sampling:} sample a fresh batch $B$ of data records
		\STATE {\bfseries Computing Percentile:}
        compute
        \[
        \tau = \text{lower-percentile}_\alpha
        \left(\left\{\ell(x;\theta_t): x \in B\right\}\right),
        \]
        where $\ell(x;\theta_t)$ is the \textit{sum} of per-token cross-entropy loss of $\theta_t$ on record $x$
        
        \STATE {\bfseries Selection:} for each $x \in B$, add it to $B_t$ with probability
        \STATE \hspace{1em} \textbf{LossH (Head):} 1 if $\ell(x;\theta_t)\leq \tau$ and 0 otherwise. 
        \STATE \hspace{1em} \textbf{LossHF (Head-Flattened):} $\frac{\ell(x;\theta_t)}{\tau}$ if $\ell(x;\theta_t)\leq \tau$ and 0 otherwise. 
        
        \STATE {\bfseries Batch Accumulation:} if $|B_t| < b$, repeat sampling, percentile computation, and selection
        \STATE {\bfseries Selective Training:} update $\theta_t$ on the first $b$ records in $B_t$ to obtain $\theta_{t+1}$
	\end{algorithmic}
\end{algorithm}

\input{figures/fig_data_usage}
\subsection{Semi-Synthetic Validation: Reaching the Capacity Limit}

\label{sec:selection_experiment}

As a sanity check, we first evaluate our data selection schemes for pretraining on synthetic power law phonebook dataset, across different training data distributions. All selective training experiments use the same settings as \Cref{sec:cap_speed_mem_standard_training}, except that we additionally tune the data selection ratio $\alpha$. To reduce computation cost, we fix the learning rate as 5e-5,
and only tune the optimal batch-size over $\{2560, 5120, 10240\}$ and tune  $\alpha$ via grid search over $\alpha\in\{0.1, 0.2, \cdots, 1.0\}$. We emphasize that this only makes our results stronger -- our selection algorithms are able to  outperform training on full dataset, despite less extensive learning rate and batch-size tuning.

Our results are summarized in \Cref{fig:power_gap} (red curves). Our selection method consistently improves the fact accuracy to close to capacity limit, significantly improving over  training on the full dataset especially when the training dataset contains a large number of facts. See \Cref{tab:phonebook_selection_unwei} for more detailed reports of accurate fact count in each setting, and see \Cref{app:additional_selection_results} for similar improvements in another weighted notion of fact accuracy. One caveat is that the improvement is smaller under larger power law exponent. This is largely due to the approximation error in using loss as a proxy for fact entropy and frequency, as we show in \Cref{sec:ablation_oracle} through ablation comparisons between our loss-based selection \Cref{alg:cramless} and a set of oracle-aided selection algorithms that precisely knows the underlying fact distribution. We also visualize the comparison between data selected by our method versus oracle-aided methods in \Cref{fig:hist_data_usage}.

To capture more realistic real-world high-entropy facts, we also evaluate our methods on natural author-title mapping facts in the arXiv-papers~\citep{Saga} dataset. Note that this dataset is too small for pretraining from scratch, thus we instead perform LoRA finetuning of the Llama-3.2-1B pretrained model~\citep{dubey2024llama} on  facts from a subset of $171104$ arXiv papers published in 2025 after the pretrained models' cut-off dates. For this semi-synthetic dataset, we observe that \Cref{alg:cramless} also boosts fact accuracy to the capacity limit of LoRA adapters, while not inducing any additional forgetting compared to full dataset training. (See detailed results in \Cref{app:finetune_additional_selection_results}.)
This is consistent with the recent work~\citep{sanyal2025upweighting} that proposes to upweight low-loss samples in the training objective to reduce catastrophic forgetting during finetuning.

\subsection{Boosting Wikipedia Entity Fact Accuracy in Pretraining}
\label{ssec:wiki_pretraining_main}

Finally, we validate the effect of our data selection schemes in a more general setting of pretraining on real-world knowledge-intensive Wikipedia corpus.

\paragraph{Fact-Annotated Wikipedia Corpus} We use a high-quality  Wikipedia corpus (3B tokens) from \citep{zhao2025pre} that annotates factual information on top of the OLMo2 Wikipedia mix~\citep{groeneveld2024olmo}. We post-processed their annotated corpus to remove the database calls, and only preserve the annotated boundaries of facts in the original Wikipedia articles. Each  of our training data records takes the following format.
\begin{lstlisting}
  Pterostylis stricta was first described in <|start_of_fact|>1972<|end_of_fact|> by <|start_of_fact|>Stephen Clemesha and Bruce Gray<|end_of_fact|> and the description...
\end{lstlisting}
We view each (context, fact tokens) pair in this dataset as one fact. E.g., in the above record, the first fact is $(Q_1, A_1)$ where $Q_1$ is the context string ``Pterostylis stricta was first described in '' and the answer $A_1$ is the string ``\texttt{<|start\_of\_fact|>1972<|end\_of\_fact|>}''. We use the same train, test and validation splits as \citep{zhao2025pre}, and show the number of records and facts in each split in \Cref{tab:fact_counts}, where on average each record contains around 10 facts.

\paragraph{Adapting Our Selection Algorithm to General Corpus} Our selection algorithm \Cref{alg:cramless} operates at the unit of fact, i.e., it implicitly assumes each record $x$ corresponds to one fact. In the general pretraining corpus, however, facts and non-fact natural language are mixed in each record (Wikipedia article), and thus requires modifying \Cref{alg:cramless} to correctly operate at the fact-level, while not affecting the learning of the remaining non-fact parts of the data record. To this end, we use a variant of \Cref{alg:cramless} for our Wikipedia experiments, presented in \Cref{alg:cramless_wiki}. To avoid adversely affecting the learning of non-fact content, we keep all tokens that are not fact answers, while increasing the weight of tokens in selected fact answers to keep the relative weight between fact and non-fact tokens unchanged before and after selection. 

\begin{algorithm}[t!]
	\caption{LossH-Wiki / LossHF-Wiki Selective Training (One Step)}
	\label{alg:cramless_wiki}
	\begin{algorithmic}
		\STATE {\bfseries Input:} data selection ratio $\alpha$, current model $\theta_t$ at iteration $t$, target batch size $b$
        \STATE {\bfseries Data Sampling:} sample a fresh batch $B_t$ of $b$ data records
        \STATE {\bfseries Initialize Selection Mask:} $M \gets( b \times \text{context length})$ all-zero-matrix
		\STATE {\bfseries Computing Percentile:}
        compute
        \[
        \tau = \text{lower-percentile}_\alpha
        \left(\left\{\ell(A;\theta_t, Q): (Q, A)\in \text{fact}(B)\right\}\right),
        \]
        where $\ell(A;\theta_t, Q)$ is the \textit{sum} of per-token cross-entropy loss of $\theta_t$ on answer $A$ given context $Q$.
        
        \STATE {\bfseries Selection:} for each fact $(Q,A) \in \text{fact}(B)$, set the masks in $M$ for tokens in answer $A$ to one with probability
        \STATE \hspace{1em} \textbf{LossH-Wiki (Head):} 1 if $\ell(A;\theta_t, Q)\leq \tau$ and 0 otherwise. 
        \STATE \hspace{1em} \textbf{LossHF-Wiki (Head-Flattened):} $\frac{\ell(A;\theta_t, Q)}{\tau}$ if $\ell(A;\theta_t, Q)\leq \tau$ and 0 otherwise. 
        
        \STATE {\bfseries Upscaling token masks for selected answers:} $M\leftarrow \frac{\text{all answer tokens in }B}{\text{selected answer tokens in }B} \cdot M$
        \STATE {\bfseries Including remaining tokens:} set the masks in $M$ for all tokens not in fact answers to one.
        \STATE {\bfseries Selective Training:} update $\theta_t$ via gradient of weighted sum of per-token-loss over $B_t$ weighted by mask $M$ to obtain $\theta_{t+1}$
	\end{algorithmic}
\end{algorithm}

\paragraph{Evaluation metrics} We evaluate three sets of performance metrics: (1) fact accuracy on the test split of Wikipedia corpus, measured under stochastic decoding following \Cref{def:fact_acc}, i.e., by the probability of generating $A_i$ given $Q_i$ as context (under stochastic decoding) averaged over all annotated facts $(Q_i, A_i)_{i=1,2,\cdots}$ in the test split. High test fact accuracy shows that the model not only memorizes the facts in the training dataset, but also successfully uses that knowledge to answer facts accurately in fresh test cases. As a sanity check, we also evaluate the train fact accuracy in \Cref{fig:lmlmwiki_detail}, where the trend is similar. (2) Knowledge-MMLU accuracy measured as the average accuracy over a subset of MMLU tasks that target world knowledge, including high-school-us-history, high-school-european-history, world-religions,  clinical-knowledge, global-facts, human-aging, medical-genetics, nutrition, virology, high-school-geography, and human-sexuality.   (3) General capability accuracy  measured by the average accuracy over a set of standard NLU tasks following \citet{zhao2025pre}, including CommonsenseQA, HellaSwag, PIQA, Social IQA, ARC-Easy.

\input{lmlmwiki_loss_plots}

\paragraph{Results} We pretrain small language models with 110m parameters from scratch on the annotated Wikipedia Corpus for around 8 epochs, using the same hyperparameters as \citet{zhao2025pre}. Our results for 110m model under full data selection ($\alpha=1$) matches the performance reported for 124m model in \citet[Table 12]{zhao2025pre}. See \Cref{app:wikipedia_additional_results} for the detailed setups. 

Our main observation is that our LossH-Wiki selection \Cref{alg:cramless_wiki} improves fact accuracy (\Cref{fig:test_fact_accuracy_alpha}) and downstream MMLU-Knowledge accuracy (\Cref{fig:MMLU_accuracy_alpha}) without harming general natural language understanding (\Cref{fig:NLU_accuracy_alpha}). Namely,   \Cref{fig:test_fact_accuracy_alpha} shows that our LossH-Wiki selection significantly improves the test fact accuracy compared to full dataset training ($\alpha=1.0$), with the optimal selection ratio being $\alpha=0.2$. In  \Cref{fig:MMLU_accuracy_alpha}, we observe similar improvement in  MMLU-Knowledge accuracy under dataset selection ratio $\alpha=0.5$ compared to full dataset training ($\alpha=1.0$), although the improvement is noisier and smaller in scale, potentially due to the limited number of evaluation questions in MMLU subtasks that target world knowledge. As a sanity check, we also evaluate on the full MMLU benchmark in \Cref{fig:lmlmwiki_detail}, where the signal becomes even more noisy, potentially because the remaining subtasks in MMLU require a mix of knowledge and reasoning capabilities, thus diluting the fact accuracy improvements coming from our selection algorithms. We comment that high variation of Q\&A accuracy is widely observed in the literature, and it is a long-standing challenge to evaluate knowledge capabilities of language model in an accurate and stable manner. Lastly, for natural language understanding tasks,  \Cref{fig:NLU_accuracy_alpha} validates that our data selection does not harm the general capabilities of trained model: NLU accuracy remains stable across different selection ratios, except for an exceedingly small selection ratio $\alpha=0.1$.

To understand the scale of improvement, we additionally pretrain two larger models with 335m and 1.3B parameters on the full annotated Wikipedia Corpus using the same hyperparameters, and show their performances in \Cref{fig:lmlmwiki_loss} (dashed lines). As expected, these larger models generally have stronger performances than the small 110m model that is also trained on the full Wikipedia Corpus (i.e., $\alpha=1.0$). However, this gap shrinks significantly under our data selection schemes: the test fact accuracy of 110m models trained with selection ratio $\alpha=0.2$ matches the test fact accuracy of a 10X larger 1.3B model trained on full dataset, and the MMLU-Knowledge accuracy of 110m model trained with selection ratio $\alpha=0.5$ is within standard deviation of that of a 3X larger 335m model trained on full dataset. This shows that the improvements coming from our data selection algorithms are significant.

\paragraph{Discussion on Additional Baselines} We now discuss how our selection \cref{alg:cramless_wiki} compares to standard baseline selection methods in the literature, such as random pruning and deduplication. We performed random fact-level pruning experiments in \cref{tab:random_vs_ours} confirming that naive random selection consistently underperforms our method: random selection can only match full dataset training at smaller selection ratios, but cannot improve fact accuracy. Another commonly considered baseline is training data deduplication. To this end, we note that we used an existing high-quality Wikipedia corpus~\cite{zhao2025pre,groeneveld2024olmo}, which is already deduplicated following standard procedures. Thus our gains are on top of those achieved by deduplication. 
\begin{table}[htbp]
\centering
\caption{LossH-Wiki \cref{alg:cramless_wiki} (Ours) vs.\ Random Baseline for pretraining 110m model on annotated Wikipedia corpus. The \textbf{Random baseline} selects a fixed $\alpha$ fraction of facts via content hash (same facts every epoch), keeping all non-fact tokens, with $1/\alpha$ rescaling of selected fact tokens. }
\label{tab:random_vs_ours}
\begin{tabular}{ccccc}
\toprule
$\alpha$ & Random Fact-Acc & Ours Fact-Acc & Random NLU & Ours NLU \\
\midrule
0.1 & 600.55 & \textbf{835.81} (5.78) & 0.3583 & 0.3606 \\
0.2 & 647.33 & \textbf{929.25} (3.49) & 0.3602 & 0.3604 \\
0.3 & 657.25 & \textbf{914.87} (12.18) & 0.3643 & 0.3635 \\
0.4 & 666.84 & \textbf{880.15} (7.26) & 0.3593 & 0.3598 \\
0.5 & 676.29 & \textbf{851.69} (8.41) & 0.3643 & 0.3607 \\
0.6 & 687.23 & \textbf{829.15} (7.54) & 0.3597 & 0.3629 \\
0.7 & 685.46 & \textbf{790.76} (7.62) & 0.3580 & 0.3646 \\
0.8 & 687.72 & \textbf{763.60} (6.67) & 0.3623 & 0.3613 \\
0.9 & 690.86 & \textbf{733.23} (4.89) & 0.3638 & 0.3633 \\
1.0 & 696.11 (10.28) & 696.11 (10.28) & 0.3636 (0.0036) & 0.3636 (0.0036) \\
\bottomrule
\end{tabular}
\end{table}

%% file: figures/fig_power_gap.tex
\begin{figure*}
\centering
\begin{subfigure}[b]{5.6cm}
\centering
\scalebox{0.45}{\begin{tikzpicture}
\begin{axis}[
    width=10cm,
    height=8cm,
    xlabel={\# Facts in Training Dataset},
    ylabel={Accurate Fact Count (millions)},
    xmin=0, xmax=10240000,
    xtick={0, 5120000, 10240000},
    xticklabels={0, 5.12m, 10.24m},
    ymin=0, ymax=3500000,
    ytick={0, 1280000, 2560000},
    yticklabels={0, 1.28m, 2.56m},
    scaled ticks=false,
    grid=major,
    legend pos=south east,
    legend style={font=\Large},
    label style={font=\Large},
    tick label style={font=\Large},
    mark size=2pt,
]

\addplot [red, mark=*, thick] table[col sep=comma,x expr=\thisrow{num sample},y expr=\thisrow{loss_truncate_head}*\thisrow{num sample}] {evals/figure_data/df_facts_mem_gap_uniform_filter.csv};
\addlegendentry{Our LossHF}

\addplot [blue, mark=square*, thick] table[col sep=comma,x=x,y=y_short] {evals/figure_data/df_facts_mem_gap_uniform.csv};
\addlegendentry{Full Dataset}

\addplot[gray, thick] table[x expr=\thisrow{x}*0.693/50.656, y expr=\thisrow{y}*0.693/50.656, col sep=comma]{evals/figure_data/df_bits_mem_gap_horizontal_line.csv};
\addlegendentry{Capacity limit}

\addplot[black, dashed, thick] table[x expr=\thisrow{x}*0.693/50.656, y expr=\thisrow{y}*0.693/50.656, col sep=comma] {evals/figure_data/df_bits_mem_gap_diagonal_line.csv};
\addlegendentry{$y=x$}

\end{axis}
\end{tikzpicture}}
\caption{Uniform}
\label{fig:fact_acc_gap_uniform}
\end{subfigure}%
\begin{subfigure}[b]{5.6cm}
\centering
\scalebox{0.45}{\begin{tikzpicture}
\begin{axis}[
    width=10cm,
    height=8cm,
    xlabel={Num. Facts in Training},
    xmin=0, xmax=10240000,
    xtick={0, 5120000, 10240000},
    xticklabels={0, 5.12m, 10.24m},
    ymin=0, ymax=3500000,
    ytick={0, 1280000, 2560000},
    yticklabels={},
    scaled ticks=false,
    grid=major,
    label style={font=\Large},
    tick label style={font=\Large},
    mark size=2pt,
]

\addplot [red, mark=*, thick] table[col sep=comma,x expr=\thisrow{num sample},y expr=\thisrow{loss_truncate_head}*\thisrow{num sample}] {evals/figure_data/df_facts_mem_gap_power_0_5_filter.csv};

\addplot [blue, mark=square*, thick] table[col sep=comma,x=x,y=y_short] {evals/figure_data/df_facts_mem_gap_L12D768H12_power_0_5.csv};

\addplot[gray, thick] table[x expr=\thisrow{x}*0.693/50.656, y expr=\thisrow{y}*0.693/50.656, col sep=comma]{evals/figure_data/df_bits_mem_gap_horizontal_line.csv};

\addplot[black, dashed, thick] table[x expr=\thisrow{x}*0.693/50.656, y expr=\thisrow{y}*0.693/50.656, col sep=comma] {evals/figure_data/df_bits_mem_gap_diagonal_line.csv};

\end{axis}
\end{tikzpicture}}
\caption{Power law (exponent 0.5)}
\label{fig:fact_acc_gap_power_0_5}
\end{subfigure}%
\begin{subfigure}[b]{5.6cm}
\centering
\scalebox{0.45}{\begin{tikzpicture}
\begin{axis}[
    width=10cm,
    height=8cm,
    xlabel={Num. Facts in Training},
    xmin=0, xmax=10240000,
    xtick={0, 5120000, 10240000},
    xticklabels={0, 5.12m, 10.24m},
    ymin=0, ymax=3500000,
    ytick={0, 1280000, 2560000},
    yticklabels={},
    scaled ticks=false,
    grid=major,
    label style={font=\Large},
    tick label style={font=\Large},
    mark size=2pt,
]

\addplot [red, mark=*, thick] table[col sep=comma,x expr=\thisrow{num sample},y expr=\thisrow{loss_truncate_head}*\thisrow{num sample}] {evals/figure_data/df_facts_mem_gap_power_1_0_filter.csv};

\addplot [blue, mark=square*, thick] table[col sep=comma,x=num_sample,y=y_short] {evals/figure_data/df_facts_mem_gap_L12D768H12_power_1_0.csv};

\addplot[gray, thick] table[x expr=\thisrow{x}*0.693/50.656, y expr=\thisrow{y}*0.693/50.656, col sep=comma]{evals/figure_data/df_bits_mem_gap_horizontal_line.csv};

\addplot[black, dashed, thick] table[x expr=\thisrow{x}*0.693/50.656, y expr=\thisrow{y}*0.693/50.656, col sep=comma] {evals/figure_data/df_bits_mem_gap_diagonal_line.csv};

\end{axis}
\end{tikzpicture}}
\caption{Power law (exponent 1.0)}
\label{fig:fact_acc_gap_power_1_0}
\end{subfigure}
\caption{Fact accuracy (in accurate fact count) of a 110m parameter model pretrained on power-law distributed phonebook facts. Our LossHF selection \Cref{alg:cramless} boosts fact accuracy to close to capacity limit, significantly outperforming training on the full dataset. See \Cref{tab:phonebook_selection_unwei} for detailed performance. Accurate fact count capacity limit is computed following \Cref{cor:fact_mem_upper} as follows: $(2\text{bits/param}) \times (110\text{M params}) / (22 \times \log_2(10) \text{ bits/fact}) = 3.01\text{m facts}$}
\label{fig:power_gap}
\end{figure*}
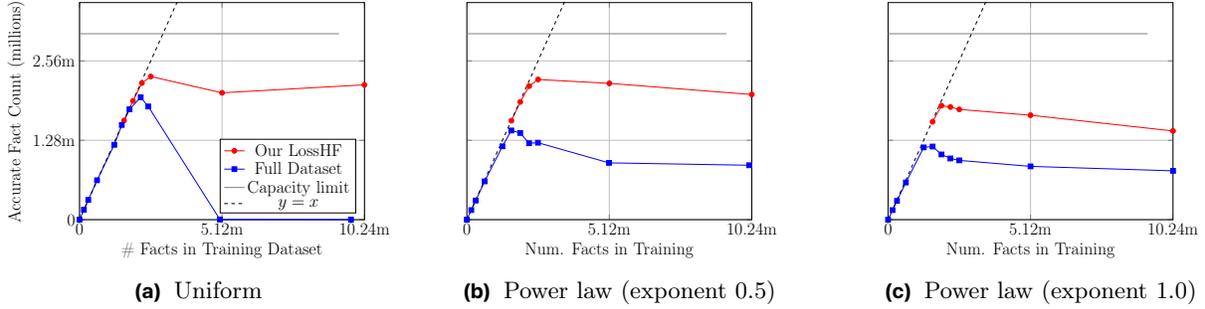

%% file: figures/fig_data_usage.tex
\begin{figure}[htbp]
\vspace{0.3cm}
\begin{subfigure}[b]{0.42\textwidth}
\centering
\includegraphics[width=\textwidth]{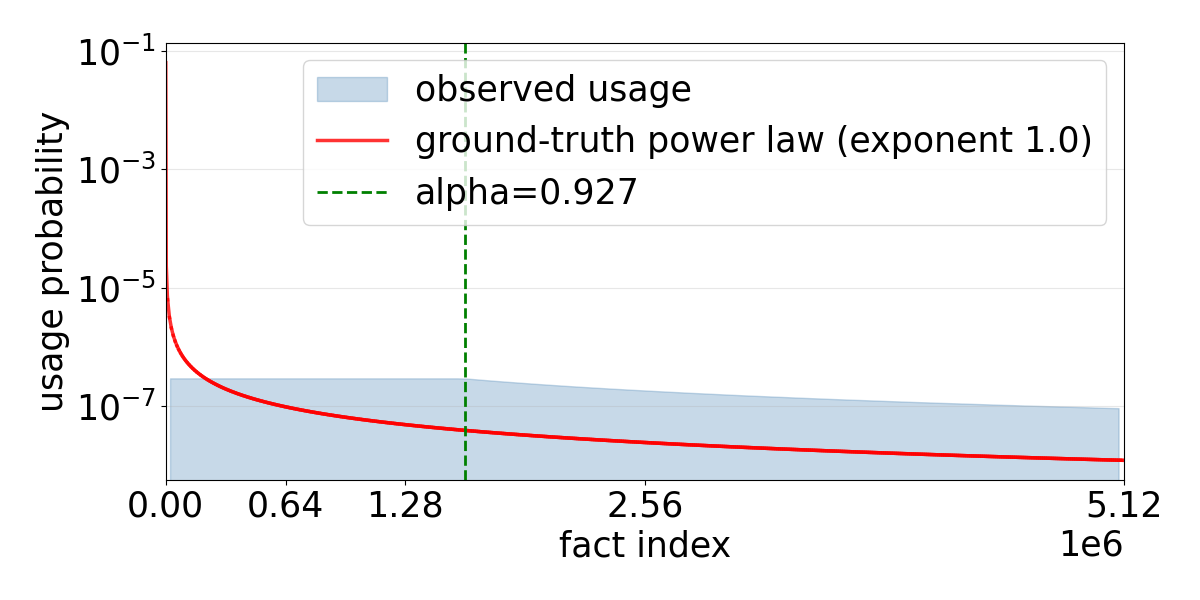}
\caption{Flattened}
\label{fig:hist_truncate}
\end{subfigure}
\hspace{1cm}
\begin{subfigure}[b]{0.42\textwidth}
\centering
\includegraphics[width=\textwidth]{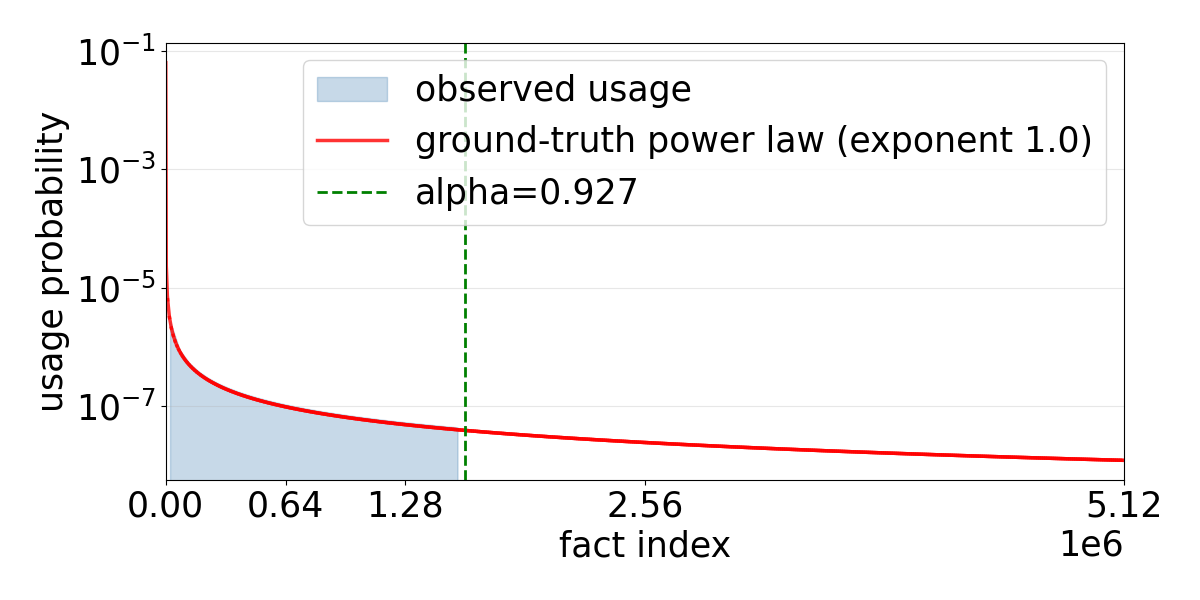}
\caption{Head}
\label{fig:hist_head}
\end{subfigure}
\hfill
\begin{subfigure}[b]{0.42\textwidth}
\centering
\includegraphics[width=\textwidth]{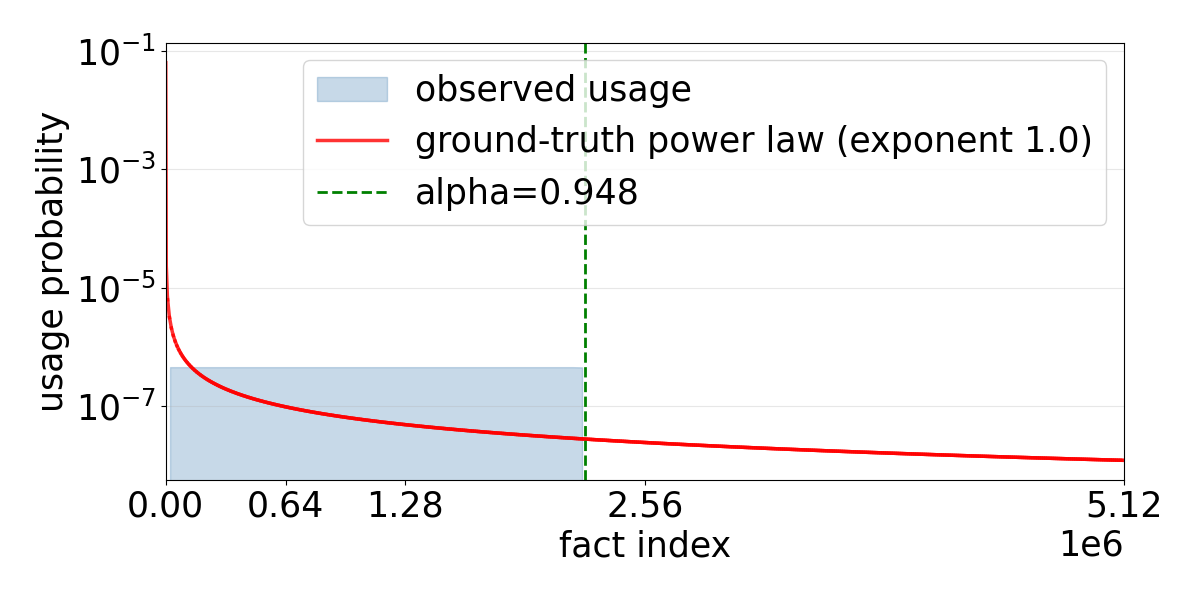}
\caption{Head-Flattened}
\label{fig:hist_truncate_head}
\end{subfigure}
\hspace{1cm}
\centering
\begin{subfigure}[b]{0.42\textwidth}
\centering
\includegraphics[width=\textwidth]{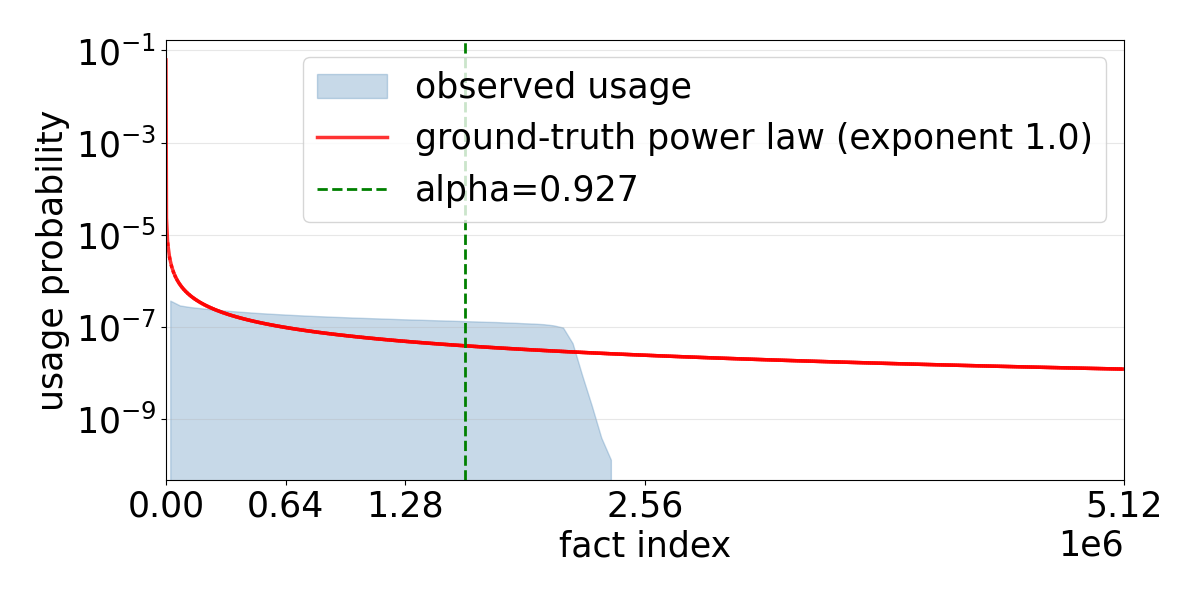}
\caption{LossHF (\Cref{alg:cramless})}
\label{fig:hist_cramless}
\end{subfigure}
\caption{Examples of normalized data usage histogram under different selection algorithms (including the oracle-aided methods described in \Cref{sec:ablation_oracle} and our loss-based selection \Cref{alg:cramless}) for training on power-law distributed synthetic phonebook datasets. All settings consider power law exponent $1.0$, number of facts in the training dataset $5120000$, and selection ratio $\alpha=0.927$ for Flattened,  Head, and LossHF \Cref{alg:cramless} and $\alpha=0.948$ for Head-Flattened.}
\label{fig:hist_data_usage}
\end{figure}

%% file: lmlmwiki_loss_plots.tex
\pgfplotstableread[col sep=comma]{lmlmwiki_loss_mean_std_110m.csv}\datatableA
\pgfplotstableread[col sep=comma]{lmlmwiki_loss_mean_std_335m.csv}\datatableB
\pgfplotstableread[col sep=comma]{lmlmwiki_loss_mean_std_1_3B.csv}\datatableC

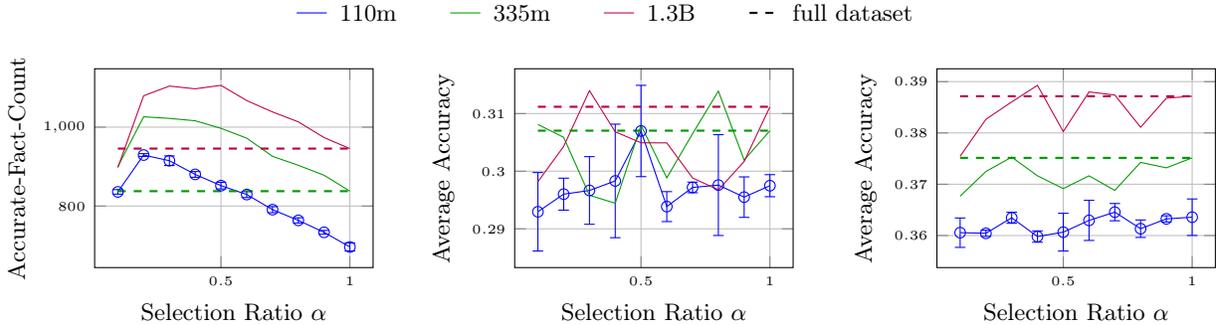
\begin{figure*}[htbp]
\centering

\begin{tikzpicture}[baseline=-0.5ex]
\draw[blue] (0,0) -- (0.4,0);
\node[right] at (0.45,0) {\small 110m};
\end{tikzpicture}
\hspace{1em}
\begin{tikzpicture}[baseline=-0.5ex]
\draw[green!60!black] (0,0) -- (0.4,0);
\node[right] at (0.45,0) {\small 335m};
\end{tikzpicture}
\hspace{1em}
\begin{tikzpicture}[baseline=-0.5ex]
\draw[purple] (0,0) -- (0.4,0);
\node[right] at (0.45,0) {\small 1.3B};
\end{tikzpicture}
\hspace{1em}
\begin{tikzpicture}[baseline=-0.5ex]
\draw[dashed, thick] (0,0) -- (0.4,0);
\node[right] at (0.45,0) {\small full dataset};
\end{tikzpicture}

\vspace{0.2cm}

\begin{subfigure}[t]{0.32\textwidth}
\centering
\begin{tikzpicture}
\begin{axis}[
   width=\textwidth,
   height=0.8\textwidth,
   xlabel={Selection Ratio $\alpha$},
   ylabel={Accurate-Fact-Count},
   grid=major,
   tick label style={font=\tiny},
   label style={font=\small},
]
\addplot[mark=o, blue, error bars/.cd, y dir=both, y explicit] table[x=fil_rate, y expr=\thisrow{test_perfect_fact_full_acc_mean}*2389+\thisrow{test_missed_corrected_fact_full_acc_mean}*4746, y error expr=\thisrow{test_perfect_fact_full_acc_std}*2389+\thisrow{test_missed_corrected_fact_full_acc_std}*4746] {\datatableA};
\addplot[no markers, green!60!black] table[x=fil_rate, y expr=\thisrow{test_perfect_fact_full_acc_mean}*2389+\thisrow{test_missed_corrected_fact_full_acc_mean}*4746] {\datatableB};
\addplot[no markers, purple] table[x=fil_rate, y expr=\thisrow{test_perfect_fact_full_acc_mean}*2389+\thisrow{test_missed_corrected_fact_full_acc_mean}*4746] {\datatableC};
\addplot[green!60!black, dashed, thick, no marks, domain=0.1:1.0, samples=2] {837.61};
\addplot[purple, dashed, thick, no marks, domain=0.1:1.0, samples=2] {945.62};
\end{axis}
\end{tikzpicture}
\caption{Fact-Accuracy (in Accurate-Fact-Count) on Test Split (containing 7135 facts)}\label{fig:test_fact_accuracy_alpha}
\end{subfigure}
\hfill
\begin{subfigure}[t]{0.32\textwidth}
\centering
\begin{tikzpicture}
\begin{axis}[
   width=\textwidth,
   height=0.8\textwidth,
   xlabel={Selection Ratio $\alpha$},
   ylabel={Average Accuracy},
   grid=major,
   tick label style={font=\tiny},
   label style={font=\small},
]
\addplot[mark=o, blue, error bars/.cd, y dir=both, y explicit] table[x=fil_rate, y=MMLU_knowledge_mean, y error=MMLU_knowledge_std] {\datatableA};
\addplot[no markers, green!60!black] table[x=fil_rate, y=MMLU_knowledge_mean] {\datatableB};
\addplot[no markers, purple] table[x=fil_rate, y=MMLU_knowledge_mean] {\datatableC};
\addplot[green!60!black, dashed, thick, no marks, domain=0.1:1.0, samples=2] {0.307030};
\addplot[purple, dashed, thick, no marks, domain=0.1:1.0, samples=2] {0.311177};
\end{axis}
\end{tikzpicture}
\caption{MMLU-Knowledge Accuracy}\label{fig:MMLU_accuracy_alpha}
\end{subfigure}
\hfill
\begin{subfigure}[t]{0.32\textwidth}
\centering
\begin{tikzpicture}
\begin{axis}[
   width=\textwidth,
   height=0.8\textwidth,
   xlabel={Selection Ratio $\alpha$},
   ylabel={Average Accuracy},
   grid=major,
   tick label style={font=\tiny},
   label style={font=\small},
]
\addplot[mark=o, blue, error bars/.cd, y dir=both, y explicit] table[x=fil_rate, y=general_mean, y error=general_std] {\datatableA};
\addplot[no markers, green!60!black] table[x=fil_rate, y=general_mean] {\datatableB};
\addplot[no markers, purple] table[x=fil_rate, y=general_mean] {\datatableC};
\addplot[green!60!black, dashed, thick, no marks, domain=0.1:1.0, samples=2] {0.375138};
\addplot[purple, dashed, thick, no marks, domain=0.1:1.0, samples=2] {0.387152};
\end{axis}
\end{tikzpicture}
\caption{NLU-Accuracy (CommonsenseQA, HellaSwag, PIQA, SIQA, ARC-Easy)}\label{fig:NLU_accuracy_alpha}
\end{subfigure}

\caption{Performance versus selection ratio $\alpha$ for using our LossH-Wiki data selection \cref{alg:cramless_wiki} for pretraining on annotated Wikipedia Corpus (3B tokens) with 66k steps and batch-size 320 (roughly 8 epochs). For the 110m model, we repeat 3 runs and report mean and standard deviation (shown in error bars). Dashed lines show full dataset training baselines for 335m and 1.3B for ease of comparison.}
\label{fig:lmlmwiki_loss}
\end{figure*}

%% file: conclusion.tex
\section{Conclusion}

In summary, our findings suggest that rather than indiscriminately scaling data, practitioners should curate training sets to match model capacity. This has implications for both efficiency (smaller models achieving competitive fact accuracy) and sustainability (reduced computational requirements for training models to memorize facts).

Our work suggests several directions for future research. In gradient-based training, memorization of new information often harms the memorization of old information, i.e., causes catastrophic forgetting. It is conceivable that by adapting our selection schemes to alleviate forgetting in training, e.g., via data replay~\citep{buzzega2020dark,verwimp2021rehearsal,li2025tic}, the speed of fact memorization could be further improved. We only consider standalone dense transformer models throughout the paper and an interesting direction is to boost fact memorization capacity and efficiency via adapting model architectures, e.g., via  mixture-of-experts (MoE)~\citep{jiang2024mixtral,jelassi2024mixture} or specialized memory architectures~\citep{cheng2026conditional,pouransari2025pretraining,weston2014memory}. Finally, our selection operates at the fact-level, and requires knowledge for the fact boundaries (such as annotation for fact tokens in the Wikipedia Corpus provided by \citet{zhao2025pre}),  while many other real-world datasets may not have clear fact formats or boundaries. It remains a challenge to design the right selection unit for boosting fact memorization on more general real-world datasets. That being said, data annotation can typically be done via a lightweight language model. In particular, it is easy to scale with the training data. Some additional recent work showing the scalability of annotation can be found in \citep{min2023factscore,maini2024rephrasing,su2025nemotron,rathi2026shaping}.

%% file: acks.tex
\section*{Acknowledgements}

The authors thank David Grangier and Skyler Seto for the help in setting up pretraining experiments, Szilvia Ujvary, Richard He Bai and Bowen Jin for insightful discussions on the literature,  Satyen Kale for valuable discussions and feedbacks on the optimization setup, and Skyler Seto and Hilal Asi for helpful suggestions on earlier drafts. Jiayuan Ye is supported by the Apple Scholars in AI/ML PhD fellowship.

%% file: app_main.tex
\input{app_additional_mem_def}
\input{app_what_does_not_work_for_est_mem}

\input{app_suboptimal_fact_mem}

\input{app_additional_selection_results}

\input{ablations}

%% file: app_additional_mem_def.tex
\section{Additional Discussions on Memorization Definitions}
\label{app:additional_mem_def_discussion}

\subsection{Deferred Proof for Fact Accuracy Capacity Limit}
\label{app:proof_acc_based_mem_lower}
We first show how to prove the following proposition that says memorization is upper bounded by the size of the algorithm's output space.
\begin{proposition}[Fact Memorization Capacity Limit under Fixed Model Capacity]
    For any facts $(Q_i, A_i)_{i=1}^N$, any meta prior  $\Psi$ and any dataset size $n$, the fact memorization of a learning algorithm $\mathcal{A}$ with discrete output model parameters space $\mathcal{W}$ satisfies the following upper bound.
    \begin{align}
        \underset{(Q_i, A_i)_{i=1}^N}{\text{Mem}}\left(\mathcal{A}; \Psi, n\right)\leq \ln|\mathcal{W}| \label{eqn:axiom_mem}
    \end{align}
    where $|\mathcal{W}|$ is the cardinality of discrete model parameters space $\mathcal{W}$. \label[proposition]{prop:axiom_mem}
\end{proposition}
\begin{proof}
    By definition, 
    \begin{align}
        \underset{(Q_i, A_i)_{i=1}^N}{\text{Mem}}\left(\mathcal{A}; \Psi, n\right) = & I((A_1(\theta),\cdots,A_N(\theta)),\mathcal{A}(D))\nonumber\\
        = & H(\mathcal{A}(D)) - H(\mathcal{A}(D)|A_1(\theta),\cdots,A_N(\theta)) \nonumber\\
        \leq & H(\mathcal{A}(D)) \leq \ln|\mathcal{W}|\nonumber
    \end{align}
    where the last inequality is by the fact that the maximal entropy distribution over discrete space $\mathcal{W}$ is the uniform distribution with entropy $\ln|\mathcal{W}|$.
\end{proof}

We now prove \Cref{thm:acc_based_mem_lower} for relating fact memorization to per-fact accuracy.

\begin{theorem}[Fact Memorization Lower Bound by Per-fact accuracy]\label{appthm:acc_based_mem_lower}
    For any facts $(Q_i, A_i)_{i=1}^N$, any meta prior~$\Psi$ and any dataset size $n$, we have 
    \begin{align}
    \underset{(Q_i, A_i)_{i=1}^N}{\text{Mem}}\left(\mathcal{A}; \Psi, n\right)
    \geq
    {H}\left[(A_1(\theta), \cdots, A_N(\theta))\right]
    - \sum_{i=1}^N\left(
   \Pr\left[I_i=0\right] 
    \cdot
    {H}\left[A_i(\theta)\mid I_i=0\right] + {H}\left[I_i\right]\right)
    \end{align}
    where $I_i = \mathbf{1}_{f(\mathcal{A}(D);Q_i)= A_i(\theta)}$ is the accuracy indicator of the trained model on fact $i$, $f(\mathcal{A}(D);Q_i)$ denotes the prediction by the trained model  $\mathcal{A}(D)$ on question $Q_i$, and the entropy $H$ and probability $\Pr$ are over the randomness of $\mathcal{A}$, $\theta\sim\Psi$, and $D\sim\mathcal{P}_\theta^n$.
\end{theorem}
\begin{proof} 
    By definition, we have
    \begin{align}
        \underset{(Q_i, A_i)_{i=1}^N}{\text{Mem}}\left(\mathcal{A}; \Psi, n\right) = &  I((A_1(\theta), \cdots, A_N(\theta)), \mathcal{A}(D))\nonumber\\
        = &   {H}\left[(A_1(\theta), \cdots, A_N(\theta))\right]  - {H}\left[A_1(\theta), \cdots, A_N(\theta)|\mathcal{A}(D)\right] \label{eqn:use_independence_fact}\\
        \geq & H\left[(A_1(\theta), \cdots, A_N(\theta))\right]   - \sum_{i=1}^N  {H}\left[A_i(\theta)|\mathcal{A}(D)\right] \label{eqn:use_entropy_upper_sum_fact}
    \end{align}
    where \eqref{eqn:use_independence_fact} is by the definition of mutual information, and \eqref{eqn:use_entropy_upper_sum_fact} is by the sub-additivity of entropy, i.e., $H\left[X_1, \cdots, X_n|Y\right]\leq \sum_{i=1}^nH\left[X_i|Y\right]$ for any random variables $X_1, \cdots, X_n, Y$. By further applying Fano's inequality to \eqref{eqn:use_entropy_upper_sum_fact}, we prove that 
    \begin{align}
         \underset{(Q_i, A_i)_{i=1}^N}{\text{Mem}}\left(\mathcal{A}; \Psi, n\right) \geq &   {H}\left[(A_1(\theta), \cdots, A_N(\theta))\right] - \sum_{i=1}^N\Bigg( \Pr\left[I_i=0\right] \cdot  {H}\left[A_i(\theta)|I_i=0\right] + {H}\left[I_i\right]  \Bigg) 
    \end{align}
    where $I_i = \mathbf{1}_{f(\mathcal{A}(D);Q_i)= A_i(\theta)}$ is the accuracy indicator of the trained model in predicting fact $i$. 
\end{proof}

We finally provide proof for \Cref{cor:fact_mem_upper}.

\begin{corollary}[Fact Accuracy Capacity Limit on Fixed-Entropy  Random Facts]  As a special case of \Cref{thm:acc_based_mem_lower}, if each answer follows uniform distribution over a discrete answer domain $\mathcal{M}_i$ with $\ln|\mathcal{M}_i|=b$, and if $A_1(\theta), A_2(\theta), \cdots, A_N(\theta)$ are independent over the meta prior $\theta\sim\Psi$,    then the accurate fact count of any learning algorithm $\mathcal{A}$ satisfies
    \begin{align}
        \underset{\theta\sim\Psi}{\mathbb{E}}\left[\underset{(Q_i, A_i)_{i=1}^N}{\text{Acc-Cnt}}\left(\mathcal{A}; \theta, n\right)\right] \leq \frac{\ln|\mathcal{W}| + N\ln 2}{b} \label{appeqn:fact_mem_upper}
    \end{align}
    \label[corollary]{appcor:fact_mem_upper}
\end{corollary}
\begin{proof}
    As a special case of \cref{appthm:acc_based_mem_lower}, when $\ln|\mathcal{M}_i|=b$, the term $\underset{\theta\sim\Psi}{H}\left[A_i(\theta)|I_i=0\right]\leq b$ due to the fact that the maximum entropy distribution over a discrete space $\mathcal{M}_i$ is the uniform distribution with entropy $\ln|\mathcal{M}_i| = b$. Thus by plugging $\underset{\theta\sim\Psi}{H}\left[A_i(\theta)|I_i=0\right]\leq b$ and the independence condition that enabled $H\left[(A_1(\theta),\cdots,A_N(\theta)\right] = \sum_{i=1}^NH\left[A_i(\theta)\right] = Nb$ into \eqref{appthm:acc_based_mem_lower} we obtain that
    \begin{align}
        \underset{(Q_i, A_i)_{i=1}^N}{\text{Mem}}\left(\mathcal{A}; \Psi, n\right) \geq & \sum_{i=1}^N\Bigg( b \cdot  \Pr\left[I_i=1\right] - {H}\left[I_i\right]  \Bigg) \nonumber\\
        \geq &\sum_{i=1}^N\Bigg( b \cdot \Pr\left[I_i=1\right] - \ln 2  \Bigg) \nonumber\\
        = & b \cdot \underset{\theta\sim\Psi}{\mathbb{E}}\left[\underset{(Q_i, A_i)_{i=1}^N}{\text{Acc-Cnt}}\left(\mathcal{A}; \theta, n\right)\right] - \ln 2 \cdot N \label{eqn:proof_1_fact_mem_lower_by_per_fact_acc}
    \end{align}
    where the second-to-last inequality is by observing that ${H}\left[I_i\right]\leq \ln 2$  due to the fact that the maximum entropy distribution over discrete space $I_i\in\{0,1\}$ is the uniform distribution with entropy $\ln 2$, and the last equality is by observing that $\underset{\theta\sim\Psi}{\mathbb{E}}\left[\underset{(Q_i, A_i)_{i=1}^N}{\text{Acc-Cnt}}\left(\mathcal{A}; \theta, n\right)\right] = \sum_{i=1}^N\Pr\left[I_i=1\right]$ by  \Cref{def:fact_acc} of fact accuracy. Combining \Cref{eqn:proof_1_fact_mem_lower_by_per_fact_acc,eqn:axiom_mem} suffices to prove the bound in \Cref{appeqn:fact_mem_upper}
\end{proof}

\subsection{Prior Loss-based Lower Bounds for (Fact) Memorization}

\label{def:loss_mem_prior_works}

We now discuss prior memorization definitions and loss-based memorization lower bounds. 

\paragraph{Prior Loss-based Unintentional Memorization Lower Bound} A long line of prior works\citep{brown2021memorization,feldman2025trade,morris2025much} define memorization of a learning algorithm about its input dataset as follows.

\begin{definition}[Memorization] \label{def:mem}
     The memorization of a learning algorithm $\mathcal{A}$ (\Cref{def:learning_alg}) about its input dataset $D$ (\Cref{def:dataset}) is defined by the mutual information between the dataset $D$ and the learning algorithm's output $\mathcal{A}(D)$ as follows.
    \begin{align}
        \text{Mem}_{n, \mathcal{W}, \Psi}\left(\mathcal{A}\right) = & I(\mathcal{A}(D), D) \label{eqn:def_mem_app}
    \end{align}
    where the expectation is over random sampling of dataset $D\sim\mathcal{P}_\theta$ from training distribution parameterized by $\theta$, the sampling of data distribution parameters $\theta$ from the meta prior $\Psi$, and the randomness of the learning algorithm~$\mathcal{A}$.
\end{definition}
\vnote{This seems to be referring to the standard subadditivity of entropy similar to the one we just mentioned. Not sure what "knowing $x_i$ means. Perhaps easiest to state it as simple lower bound on total information. In general as you say later the discussion of excess memorization is not really relevant for us. So we do not need the details below and can discuss and compare in the related work. }
One simple lower bound for total memorization is the ``unintended memorization'' of a trained model $\hat{\theta}$ about its individual training data. This is the quantity estimated in many prior works~\citep{feldman2020neural,ye2023leave,morris2025much} up to translations, written as follows.

\begin{proposition}[{\citep[Proposition 1, Proposition 4, Section 2.3]{morris2025much}}]
    \begin{align}
        \text{Mem}_{n, \mathcal{W}, \Psi}\left( \mathcal{A}\right)\gtrapprox \sum_{x\in D} \left(\ell(\theta_r; x) - \ell(\hat{\theta}; x)\right)
    \end{align}
    where $\gtrapprox$ denotes approximately greater than or equal to, $\theta_r$ is a reference model trained on freshly sampled $n$ i.i.d. samples from a data distribution specified by the same underlying parameters $\theta$ as dataset $D$, and $\ell$ denotes the sum of per-token cross-entropy prediction loss.
\end{proposition}

However, this lower bound is very small when there is no unintentional memorization, i.e., all memorization are necessary for learning the true data distribution parameters (such as retrieval or in general memory-intensive Q\&A task), and is thus insufficient for our problem of fact memorization. For example, for the synthetic power law phonebook experiments in \Cref{sec:cap_speed_mem_standard_training}, the model trained on the input dataset $D$
would have similar loss as the reference model trained on fresh i.i.d. samples from the same fact distribution, as the phone numbers are fixed by the training data distribution and remain unchanged under data resampling.

\paragraph{Prior Loss-based Total Memorization Lower Bound} Another set of prior works~\citep{allen2024physics,gu2025data} propose loss-based lower bounds for \text{total memorization}, which more naturally translates to our fact-learning setting. For completeness, below we present the derivations along with translated statements.

\begin{theorem}[Memorization Lower Bound by Negative Log Likelihood (Similar to~\citep{allen2024physics,gu2025data})]      \label{thm:nll_mem_lower_bound}
    Let $(Q_i,A_i(\theta))_{i=1}^N$ be facts encoded by the data distribution parameterized by $\theta$,
    as defined by \Cref{def:facts}. For any fixed $n, \mathcal{W}, \Psi$,  we have the following loss-based memorization lower bound for any learning algorithm $\mathcal{A}$
    \begin{align}
        \underset{(Q_i, A_i)_{i=1}^N}{\text{Mem}}\left(\mathcal{A}; \Psi, n\right) \geq  &  \underset{\theta\sim\Psi}{H}\left[(A_1(\theta), \cdots, A_N(\theta))\right] +  \sum_{i=1}^N  \underset{\theta\sim\Psi, D\sim\mathcal{P}_\theta^n, \hat{\theta}\sim \mathcal{A}(D)}{\mathbb{E}}\left[\ln\left(\Pr_f[f(\hat{\theta}; Q_i)=A_i(\theta)]\right) \right] 
    \end{align}
    where $f(\hat{\theta}; Q_i)$ denotes the (possibly randomized) prediction of  the trained model $\hat{\theta}$ given prefix $Q_i$.
\end{theorem}
\begin{proof}
    By definition, 
    \begin{align}
        \underset{(Q_i, A_i)_{i=1}^N}{\text{Mem}}\left(\mathcal{A}; \Psi, n\right) = &  I((A_1(\theta), \cdots, A_N(\theta)), \mathcal{A}(D))\\
        = & \underset{\theta\sim\Psi}{H}\left[(A_1(\theta), \cdots, A_N(\theta))\right] - \underset{\theta\sim\Psi, D\sim\mathcal{P}_\theta^n, \hat{\theta}\sim \mathcal{A}(D)}{H}\left[A_1(\theta), \cdots, A_N(\theta)|\hat{\theta}\right] \label{eqn:use_independence}\\
        \geq & \underset{\theta\sim\Psi}{H}\left[(A_1(\theta), \cdots, A_N(\theta))\right] - \sum_{i=1}^N \underset{\theta\sim\Psi, D\sim\mathcal{P}_\theta^n, \hat{\theta}\sim \mathcal{A}(D)}{H}\left[A_i(\theta)|\hat{\theta}\right]\label{eqn:use_entropy_upper_sum}\\
        = & \underset{\theta\sim\Psi}{H}\left[(A_1(\theta), \cdots, A_N(\theta))\right] + \sum_{i=1}^N\ \sum_{a}\underset{\theta,\hat{\theta}}{\mathbb{E}}\left[ \Pr[A_i(\theta)=a|\hat{\theta}] \cdot \ln\left(\Pr[A_i(\theta)=a|\hat{\theta}]\right) \right] \label{eqn:use_def_cond_entropy}\\
        \geq & \underset{\theta\sim\Psi}{H}\left[(A_1(\theta), \cdots, A_N(\theta))\right] + \sum_{i=1}^N \sum_{a}\underset{\theta,\hat{\theta}}{\mathbb{E}}\left[\Pr[A_i(\theta)=a|\hat{\theta}] \cdot \ln\left(\Pr_f[f(\hat{\theta}; Q_i)=a]\right) \right] \label{eqn:use_gibbs_inequality}
    \end{align}
    where \eqref{eqn:use_independence} is by the definition of mutual information; \eqref{eqn:use_entropy_upper_sum} is by $H(X_1, \cdots, X_n|Y)\leq \sum_{i=1}^nH(X_i|Y)$ for any random variables $X_1, \cdots, X_n, Y$; \eqref{eqn:use_def_cond_entropy} is by the definition of conditional entropy; and \eqref{eqn:use_gibbs_inequality} is by using the Gibbs inequality which ensures $\sum_{a}\Pr[X=a]\cdot \ln\Pr[X=a] \geq \sum_{a}\Pr[X=a]\cdot \ln(\Pr[Y=a])$ for any (discrete) random variables $X$ and $Y$. (This is intuitively saying for describing random variable $X$, the Huffman code optimized for $X$ has the shortest length.)
    
\end{proof}

This suggests that we can lower bound memorization via the difference between the  entropy of answers generated from random prior, versus the average  negative log probability of predicting the right answer for each question on top of observing the trained model. The prediction function $f$ is specified by the decoding and answer matching process when using the trained language model. In the special case of $f$ given by vanilla stochastic-decoding answering function, \eqref{eqn:def_mem_pred} simplifies to the following loss-based lower bound for total memorization.

\begin{corollary}[Memorization Lower Bound by Loss]\label[corollary]{cor:mem_lower_loss}
     Let $\ell(A_i;\hat{\theta}, Q_i)$ be the \textit{sum} of per-token cross-entropy loss for using trained model $\hat{\theta}$ to predict answer $A_i$ given prefix $Q_i$. If the same condition of \Cref{thm:nll_mem_lower_bound}, then under stochastic decoding, we have the following loss-based memorization lower bound
    \begin{align}
        \underset{(Q_i, A_i)_{i=1}^N}{\text{Mem}}\left(\mathcal{A}; \Psi, n\right)  \geq  & \underset{\theta\sim\Psi}{H}\left[(A_1(\theta), \cdots, A_N(\theta))\right] -  \sum_{i=1}^N  \underset{\theta\sim\Psi, D\sim\mathcal{P}_\theta^n, \hat{\theta}\sim \mathcal{A}(D)}{\mathbb{E}}\left[\ell\left(A_i; \hat{\theta}, Q_i\right) \right] \label{eqn:def_mem_loss}
    \end{align}
\end{corollary}

\textbf{Estimating the first term $\underset{\theta\sim\Psi}{H}\left[(A_1(\theta), \cdots, A_N(\theta))\right] $, i.e., the joint entropy of the fact answers $A_1(\theta), \cdots, A_N(\theta)$ over $\theta\sim\Psi$, requires accurate approximations for the meta prior $\Psi$.} For a synthetically constructed dataset, we have control over the meta distribution $\psi$ and can compute entropy exactly. For real datasets however, we often do not have precise knowledge of fact entropy, nor about the (existence of) independently distributed facts.  In our experiments, we use the \textbf{median of model's training loss across the first epoch as approximations for the average of entropy over all facts $i=1,\cdots,N$ under meta prior}. (Besides this heuristic, we also tried other heuristic choices of meta distribution, such as the loss of the model at initialization, the loss of a reference model trained on disjoint facts. However, we found these choices severely overestimate memorization, as discussed in \Cref{app:wrong_meta_prior}.)

%% file: app_what_does_not_work_for_est_mem.tex
\subsection{More Discussions on Meta Prior Choices}

\label{app:wrong_meta_prior}

Below we discuss two other choices of meta prior, that we found to be overestimating memorization lower bound in our experiments. We identify them as overestimation because the estimated memorization lower bound grows indefinitely with regard to training dataset size despite constrained model size, thus contradicting \Cref{thm:nll_mem_lower_bound}.

\paragraph{Using Reference Model trained on Disjoint Facts as Meta Prior May Overestimate Memorization due to Overfitting} This is similar to the approach in \citep{morris2025much}, yet we found that it overestimates memorization in our experiments. This is because reference model trained to memorize a disjoint set of facts could overfit to those facts, resulting in increasing loss on the facts in the target model's training dataset, thus serving as a baseline that overestimates the fact-entropy and memorization about training facts in the target model.

\paragraph{Using Model Initialization as Meta Prior May Overestimate Memorization due to Correlated Format Knowledge} Another intuitive approximation in practice is to heuristically choose the meta distribution $\Psi$ as the pretrained model at random (re)initializations. However, this has the risk of breaking the independence assumption among answers $A_1(\theta), \cdots, A_N(\theta)$ (as one can imagine that the predictions given similar prefixes are correlated, even at initialization), and thus overestimating memorization lower bound. For example, for synthetically constructed phonebook dataset that consists of (name, phone number) tuples (with name being six randomly drawn alphabetical characters from $a$ to $z$, and with phone number being 22 randomly drawn digits from $0$ to $9$), the average per-token loss at initialization is as high as $8$ when the underlying true average per-token entropy of each (name, phone number) tuple is $(6\cdot \ln(26) + 22\cdot \ln(10))/(6+22) = 2.507$. This overestimation is intuitively because there exists shared knowledge among answers for different questions, e.g., about the format.

%% file: app_suboptimal_fact_mem.tex
\section{Deferred Details for Suboptimal fact accuracy Results}

\subsection{Experiment Setups for Pretraining on Synthetic Phonebook}
\label{ssec:sufficient_training_setup}

\paragraph{Pretrained Model Architecture} We consider variants of the standard GPT2-style decoder-only transformer models~\citep{vaswani2017attention} with context length 1024, sinusoidal positional embeddings, ReLU activation, and post-output layer norm. Each model has $L$ layers, $H$ heads, hidden dimension $D$, and MLP dimension $4D$ in bfloat16 precision.  To control model size, we follow Pythia~\citep{biderman2023pythia} and vary $(L, D, H)$ across (6, 512, 8), (12, 768, 12), (24, 1024, 16), (16, 2048, 8), (24, 2048, 16), (32, 2560, 32) to create a family of models with sizes ranging from 42m to 1.4B parameters.

\paragraph{Hyperparameter Tuning to Ensure Sufficiently Long Training} For all experiments, we train with auto-regressive next-token-prediction cross-entropy loss, and follow prior works~\citep{brown2020language,hoffmann2022training} to use AdamW optimizer with weight decay $0.1$ and cosine learning rate scheduler: the learning rate linearly increases from $0$ to the maximal learning rate in warm-up steps ($2.5\%$ of the training steps), and then decrease to $0.1$ times of the maximal learning rate in the remaining training steps following cosine decay.   We fix the number of training steps as 800000, and perform extensive grid search for the optimal batch-size over $\{32, 80, 160, 320, 640, 1280, 2560, 5120, 10240\}$, for the optimal learning rate over $\{1e{-5}, 5e{-5}, 1e{-4}, 5e{-4}\}$. We set the gradient clipping norm as $1.0$. To identify the best training run, for each setting, we first find runs that reach close-to-best performance (within $2\%$ multiplicative difference or $0.01$ additive gap to the smallest loss over all runs), and then identify one run among them with the fastest convergences (in terms of smallest number of training tokens seen).

\subsection{Additional Fact Memorization Analysis for Pretraining on Synthetic Phonebook}
\label{ssec:additional_fact_memorization_phonebook}
In \Cref{fig:capacity} (left), we first perform ablation experiments to validate that the language model trained on sufficiently many uniformly distributed facts consistently reach the 2 bits/parameter fact memorization capacity limit (\Cref{prop:axiom_mem}), confirming the observations in prior works~\citep{allen2024physics,morris2025much,gu2025data} and validating the effectiveness of our pretraining recipe.

\begin{figure}
    \centering
    \includegraphics[width=0.95\linewidth]{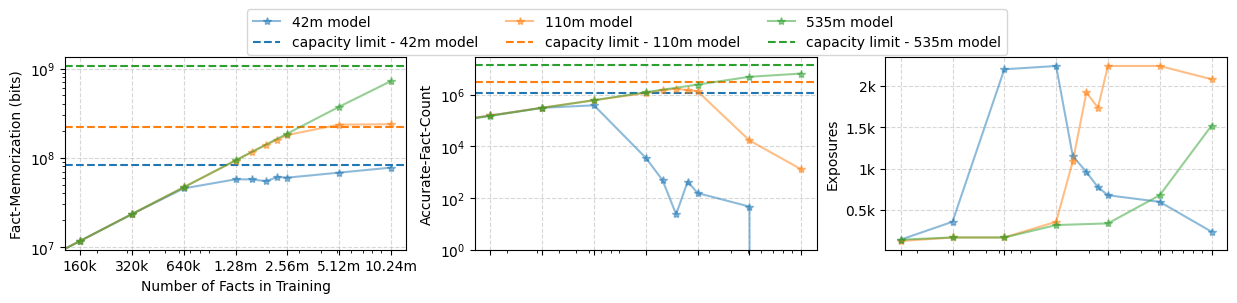}\\
    (a) varying number of facts in the training dataset\\
    \includegraphics[width=0.95\linewidth]{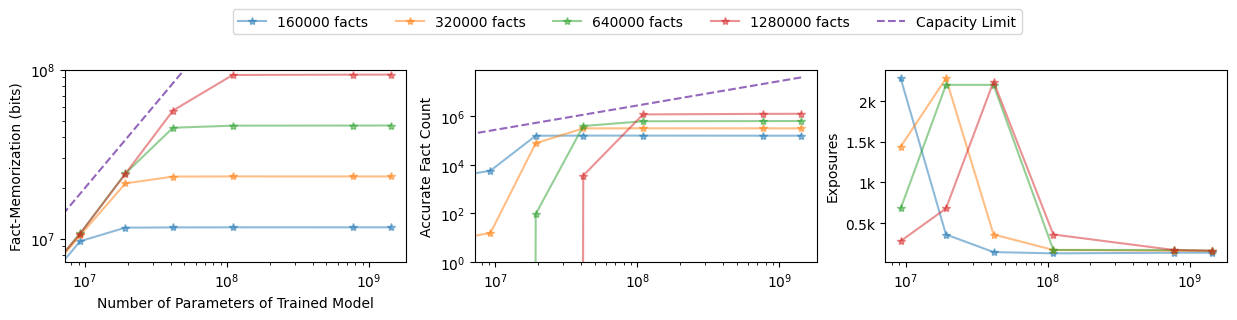}\\
    (b) varying number of parameters of the trained model
    \caption{Fact memorization (in bits), fact accuracy (in accurate fact count), and exposures (needed till convergence) for pretraining on synthetic phonebook dataset. Fact memorization is measured by loss-based lower bound (\Cref{cor:mem_lower_loss}). Fact accuracy is measured by accurate fact count defined in \Cref{def:fact_acc}, i.e., the sum of probability of generating the correct answer for each fact under stochastic decoding. Exposure is measured by the number of epochs needed till convergence for training under tuned hyperparameters (\Cref{ssec:sufficient_training_setup}). }
    \label{fig:capacity}
\end{figure}

\paragraph{Exposure Needed for Convergence Peaks at Capacity Threshold} Conditioned on the effectiveness of our pretraining strategy, we now investigate whether the fact memorization speed is a function of ratio between model size and data complexity. We perform two sets of experiments: in \Cref{fig:capacity} (a) we fix the model size and vary the training dataset size,  and in \Cref{fig:capacity} (b) we fix the training dataset and vary the model size. We observe in both cases, the exposures (i.e., the number of times that the training passes each fact) needed for convergence peaks at capacity threshold, i.e., when the model size matches the dataset size. On the one hand, this means that when model is too small to fit all training data, training longer would not help the model to converge to better optima. On the other hand, this means that larger models need fewer passes to memorize individual facts. This is intuitively because larger models are easier to optimize and has more capacity, thus reducing the interference during memorizing different facts and leading to faster convergence.

\input{figures/fig_detailed_power_gap}
\paragraph{Understanding the Exacerbated Suboptimality of Fact Accuracy under Power law Training Data} In \Cref{fig:detailed_gap_power}, we observe that the maximal fact memorization and fact accuracy drop significantly under an increasing power law exponent.
To understand the reason, we conduct two ablation experiments in \Cref{fig:detailed_gap_power}: \textbf{(1) we train a 10x bigger model} (1.4B parameters) for the same number of steps on exactly the same stream of training data; and \textbf{(2) we train a small model for 8x longer} (via increasing the number of training steps). We observe that training a 10x larger model significantly improves fact memorization and fact accuracy to near-perfect (\Cref{fig:detailed_gap_power} x-axis), while 8x longer training only yields negligible improvements (\Cref{fig:detailed_gap_power} shaded area). Thus larger models may need fewer  exposures to memorize each fact compared to small models, which is the key reason for the exacerbating fact memorization and fact accuracy gap between small and large models as the data becomes more non-uniform (i.e., as the power law exponent increases).  We remark that this is also supported by our experiments in \Cref{fig:capacity} right plots, and has also been observed for other definitions of memorization on real-world datasets in prior works~\citep{tirumala2022memorization}.

\subsection{Additional Results for Suboptimal Fact Accuracy in LoRA Finetuning}

\label{app:suboptimal_fact_mem_LoRA}

\paragraph{LoRA Finetuning has Similar Memorization Capacity to Pretraining From Scratch} We now turn to LoRA finetuning, and investigate its capacity limit for fact memorization and fact accuracy. We first repeat the experiments for synthetic phonebook dataset in the LoRA finetuning settings. As shown in \Cref{fig:phonebook_bits}, the fact memorization is also close to the 2 bits/parameter capacity limit, matching the capacity of pretraining from scratch.  This shows that the representation power of LoRA is strong enough to match full transformer model in terms of fact memorization capacity.
\input{figures/capacity_lora.tex}

\paragraph{Results for Real-World Author-Title Mapping Facts from arXiv Papers}
To capture more realistic real-world high-entropy facts, we perform LoRA finetuning of the Llama-3.2-1B pretrained model~\citep{dubey2024llama} on natural author-title mapping facts in the arXiv-papers~\citep{Saga} dataset (subselecting $171104$ articles published in 2025 after the pretrained models' cut-off dates). This is to simulate the setting of teaching a pretrained language model new knowledge that is \textit{not} in its pretraining dataset. We train on data of format "title: \_\_\_ | authors: \_\_\_" for learning the author-title mapping facts. We choose this type of fact as frontier language models tend to hallucinate author names or paper titles, which is a well-known issue in the scientific community. (E.g., GPTZero finds 100+ confirmed hallucinations in a subset of evaluated 300 NeurIPS 2025 accepted papers~\citep{gptzero_neurips2025}.)  We choose this close-to-Q\&A format because Q\&A is observed to be the most effective data format for knowledge injection during finetuning in prior works~\citep{zhao2025style}. Note that this dataset is not suitable for pretraining due to its limited size. All LoRA training runs consider a fixed context length of 64 and for hyperparameter tuning, we perform grid search for the number of training steps over $\{2000, 4000\}$, learning rate in $\{2e{-4}, 5e{-4}, 1e{-3}, 2e{-3}, 5e{-3}\}$, batch-size in $\{4000, 16000, 64000\}$.

We estimate memorized bits via loss-based memorization lower bound (\Cref{cor:mem_lower_loss}) and heuristic approximations for the average fact-entropy via median of training loss in the first epoch as discussed after \Cref{cor:mem_lower_loss}. We observe in \Cref{fig:arxiv_bits} that our estimates of memorized bits for arXiv-paper dataset is higher than that for the synthetic phonebook dataset.  This is potentially due to overestimation of memorized bits under heuristic loss-based approximations for fact-entropy, as discussed in \Cref{app:wrong_meta_prior}. We remark that accurately estimating memorized bits for real dataset (where the meta distribution is unknown or even non-existent) is a long-standing challenge in the literature~\citep{allen2024physics,morris2025much}. Nevertheless, our result still shows the promise of reaching the capacity limit of LoRA adapters for memorizing real-world high-entropy facts.

\paragraph{Fact Accuracy Drops to Close-to-Zero when Training Data Exceeds Model Capacity}
Our final observation is that similar to the results for pretraining (\Cref{fig:capacity}), fact accuracy of LoRA finetuning (on both the synthetic phonebook dataset and the more realistic arXiv-papers dataset) also drops significantly to as low as zero, as the number of facts in the training dataset increases to exceeding the model's fitting power. In \Cref{app:finetune_additional_selection_results}, we further investigate whether fact accuracy of LoRA finetuning can be similarly boosted by our data selection \Cref{alg:cramless}.

%% file: figures/fig_detailed_power_gap.tex
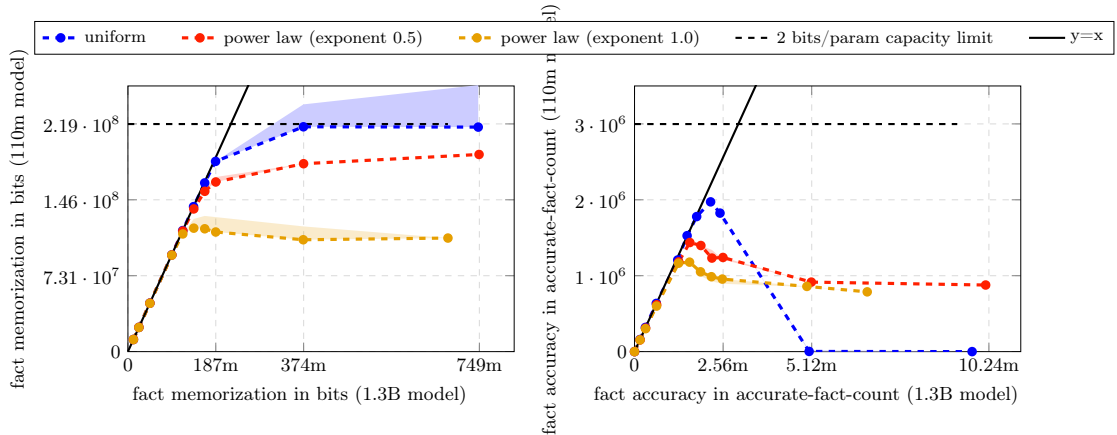
\begin{figure}
    \centering
    \scalebox{0.8}{\begin{tikzpicture}

        \pgfplotsset{
            every axis/.append style={
                width=8cm,
                height=6cm,
                grid=major,
                grid style={line width=.1pt, draw=gray!30, dashed},
                tick label style={font=\small},
                label style={font=\small},
                legend style={font=\footnotesize},
            }
        }

        \begin{axis}[
            name=plot1,
            xlabel={fact memorization in bits (1.3B model)},
            ylabel={fact memorization in bits (110m model)},
            xmin=0,
            ymin=0,
            ymax=3500000*50.656/0.693,
            xlabel style={xshift=-1em},
            xtick={0,   2560000*50.656/0.693, 5120000*50.656/0.693, 10240000*50.656/0.693},
            xticklabels={0, 187m, 374m, 749m},
            scaled ticks=false,
            ytick={0, 1000000*50.656/0.693, 2000000*50.656/0.693, 3000000*50.656/0.693},
        ]
            \addplot[name path=A, draw=none, forget plot]
                table[x=x, y=y_short, col sep=comma] {evals/figure_data/df_bits_mem_gap_uniform.csv};
            \addplot[name path=B, draw=none, forget plot]
                table[x=x, y=y_long, col sep=comma] {evals/figure_data/df_bits_mem_gap_uniform.csv};
            \addplot[fill=blue!20, draw=none, forget plot]
                fill between[of=A and B];
            \addplot[color=blue, dashed, mark=*, mark size=1.5, mark options={solid,fill=blue}, line width=1.5pt]
                table[x=x, y=y_short, col sep=comma] {evals/figure_data/df_bits_mem_gap_uniform.csv};

            \addplot[name path=C, draw=none, forget plot]
                table[x=x, y=y_short, col sep=comma] {evals/figure_data/df_bits_mem_gap_L12D768H12_power_0_5.csv};
            \addplot[name path=D, draw=none, forget plot]
                table[x=x, y=y_long, col sep=comma] {evals/figure_data/df_bits_mem_gap_L12D768H12_power_0_5.csv};
            \addplot[fill=red!75!orange!20, draw=none, forget plot]
                fill between[of=C and D];
            \addplot[color=red!75!orange, dashed, mark=*, mark size=1.5, mark options={solid,fill=red!75!orange}, line width=1.5pt]
                table[x=x, y=y_short, col sep=comma] {evals/figure_data/df_bits_mem_gap_L12D768H12_power_0_5.csv};

            \addplot[name path=E, draw=none, forget plot]
                table[x=x, y=y_short, col sep=comma] {evals/figure_data/df_bits_mem_gap_L12D768H12_power_1_0.csv};
            \addplot[name path=F, draw=none, forget plot]
                table[x=x, y=y_long, col sep=comma] {evals/figure_data/df_bits_mem_gap_L12D768H12_power_1_0.csv};
            \addplot[fill={rgb,255:red,230;green,159;blue,0}, fill opacity=0.2, draw=none, forget plot]
                fill between[of=E and F];
            \addplot[color={rgb,255:red,230;green,159;blue,0}, dashed, mark=*, mark size=1.5, mark options={solid,fill={rgb,255:red,230;green,159;blue,0}}, line width=1.5pt]
                table[x=x, y=y_short, col sep=comma] {evals/figure_data/df_bits_mem_gap_L12D768H12_power_1_0.csv};

            \addplot[color=black, line width=1pt, no marks]
                table[x=x, y=y, col sep=comma] {evals/figure_data/df_bits_mem_gap_diagonal_line.csv};

            \addplot[color=black, dashed, line width=1pt, no marks]
                table[x=x, y=y, col sep=comma] {evals/figure_data/df_bits_mem_gap_horizontal_line.csv};
        \end{axis}
        \begin{axis}[
            name=plot2,
            at={($(plot1.east)+(2cm,0)$)},
            anchor=west,
            xlabel={fact accuracy in accurate-fact-count (1.3B model)},
            ylabel={fact accuracy in accurate-fact-count (110m model)},
            xmin=0,
            ymin=0,
            ymax=3500000,
            xlabel style={xshift=-1em},
            xtick={0, 2560000, 5120000, 10240000},
            xticklabels={0, 2.56m, 5.12m, 10.24m},
            ytick={0, 1000000, 2000000, 3000000},
            scaled ticks=false,
            legend to name=combined_legend,
            legend columns=5,
            legend style={/tikz/every even column/.append style={column sep=0.5cm}},
        ]
            \addplot[name path=G, draw=none, forget plot]
                table[x=x, y=y_short, col sep=comma] {evals/figure_data/df_facts_mem_gap_uniform.csv};
            \addplot[name path=H, draw=none, forget plot]
                table[x=x, y=y_long, col sep=comma] {evals/figure_data/df_facts_mem_gap_uniform.csv};
            \addplot[fill=blue!20, draw=none, forget plot]
                fill between[of=G and H];
            \addplot[color=blue, dashed,  mark=*, mark size=1.5, mark options={solid,fill=blue}, line width=1.5pt]
                table[x=x, y=y_short, col sep=comma] {evals/figure_data/df_facts_mem_gap_uniform.csv};
            \addlegendentry{uniform}

            \addplot[name path=I, draw=none, forget plot]
                table[x=x, y=y_short, col sep=comma] {evals/figure_data/df_facts_mem_gap_L12D768H12_power_0_5.csv};
            \addplot[name path=J, draw=none, forget plot]
                table[x=x, y=y_long, col sep=comma] {evals/figure_data/df_facts_mem_gap_L12D768H12_power_0_5.csv};
            \addplot[fill=red!75!orange!20, draw=none, forget plot]
                fill between[of=I and J];
            \addplot[color=red!75!orange, dashed, mark=*, mark size=1.5, mark options={solid,fill=red!75!orange}, line width=1.5pt]
                table[x=x, y=y_short, col sep=comma] {evals/figure_data/df_facts_mem_gap_L12D768H12_power_0_5.csv};
            \addlegendentry{power law (exponent $0.5$)}

            \addplot[name path=K, draw=none, forget plot]
                table[x=x, y=y_short, col sep=comma] {evals/figure_data/df_facts_mem_gap_L12D768H12_power_1_0.csv};
            \addplot[name path=L, draw=none, forget plot]
                table[x=x, y=y_long, col sep=comma] {evals/figure_data/df_facts_mem_gap_L12D768H12_power_1_0.csv};
            \addplot[fill={rgb,255:red,230;green,159;blue,0}, fill opacity=0.2, draw=none, forget plot]
                fill between[of=K and L];
            \addplot[color={rgb,255:red,230;green,159;blue,0}, dashed, mark=*, mark size=1.5, mark options={solid,fill={rgb,255:red,230;green,159;blue,0}}, line width=1.5pt]
                table[x=x, y=y_short, col sep=comma] {evals/figure_data/df_facts_mem_gap_L12D768H12_power_1_0.csv};
            \addlegendentry{power law (exponent $1.0$)}

            \addplot[color=black, dashed, line width=1pt, no marks] table[x expr=\thisrow{x}*0.693/50.656, y expr=\thisrow{y}*0.693/50.656, col sep=comma]{evals/figure_data/df_bits_mem_gap_horizontal_line.csv};

            \addplot[color=black, line width=1pt, no marks]
                table[x expr=\thisrow{x}*0.693/50.656, y expr=\thisrow{y}*0.693/50.656, col sep=comma] {evals/figure_data/df_bits_mem_gap_diagonal_line.csv};

            \addlegendimage{color=black, line width=1pt, no marks}
            \addlegendentry{2 bits/param capacity limit}

            \addlegendimage{color=black, dashed, line width=1pt, no marks}
            \addlegendentry{y=x}
        \end{axis}

        \node at ($(plot1.north)!0.5!(plot2.north)+(0,0.8cm)$) {\pgfplotslegendfromname{combined_legend}};

    \end{tikzpicture}}
    \caption{Gap in fact memorization in bits (left) and fact accuracy in accurate fact count (right) between  sufficiently trained small model (110m parameters) and 10X larger model (1.3B parameters). Dashed lines show results for the small model trained for 800k steps, and shaded areas (extending above the dashed lines) show the improvement obtained by training the small model for 8$\times$ longer (6.4M steps).  Each point in the curve shows result for one training dataset, and curves are plotted over increasingly large training datasets containing $\{160000, 320000, 640000, 1280000, 1600000, 1920000, 2240000, 2560000\}$ facts, following different frequency distributions, including uniform, power law with exponent 0.5, and power law with exponent 1.0. All results are for the optimal run after hyperparameter tuning as discussed in \Cref{ssec:sufficient_training_setup}. }
    \label{fig:detailed_gap_power}
\end{figure}

%% file: figures/capacity_lora.tex
\begin{figure}[htbp]
\centering
\definecolor{lora2color}{RGB}{213,94,0}
\definecolor{lora4color}{RGB}{86,180,233}
\definecolor{lora8color}{RGB}{0,158,115}
\definecolor{lora16color}{RGB}{204,121,167}
\centering
\begin{tikzpicture}
\begin{axis}[
    hide axis,
    xmin=0, xmax=1,
    ymin=0, ymax=1,
    legend style={
        draw=none,
        fill=none,
        legend columns=5,
        /tikz/every even column/.append style={column sep=0.5cm},
        font=\footnotesize
    },
    legend to name=sharedlegend
]
\addlegendimage{color=lora4color, mark=square*, thick, solid}
\addlegendentry{LoRA-4}
\addlegendimage{color=lora8color, mark=square*, thick, solid}
\addlegendentry{LoRA-8}
\addlegendimage{color=lora16color, mark=square*, thick, solid}
\addlegendentry{LoRA-16}
\addlegendimage{color=black, thick, dashed}
\addlegendentry{2 bits/param capacity limit}
\end{axis}
\end{tikzpicture}
\pgfplotslegendfromname{sharedlegend}
\\
\centering
\begin{subfigure}[b]{0.42\textwidth}
\centering
\begin{tikzpicture}
\begin{axis}[
    width=\textwidth,
    height=4.2cm,
    xlabel={\# facts in the training dataset},
    ylabel={fact memorization (bits)},
    grid=major,
    ymin=0,
    ymax=20000000,
    xmax=280000,
    xlabel style={font=\footnotesize},
    ylabel style={font=\footnotesize},
    title style={font=\footnotesize},
    tick label style={font=\footnotesize},
    xtick={0, 100000, 200000},
    xticklabels={0, 100k, 200k},
    ytick={0, 5000000, 10000000, 15000000, 20000000},
    yticklabels={0, 5m, 10m, 15m, 20m},
    scaled ticks = false
]
\addplot[color=lora4color, mark=square*, thick, solid, opacity=0.7] table[
    col sep=comma,
    x=dataset_size,
    y=best_loss_diff_times_tokens
] {log_parallel_jobs_no_wait_by_taskname/tables_bits/phonebook_lora4_constant_minus_loss.csv};
\addplot[color=lora8color, mark=square*, thick, solid, opacity=0.7] table[
    col sep=comma,
    x=dataset_size,
    y=best_loss_diff_times_tokens
] {log_parallel_jobs_no_wait_by_taskname/tables_bits/phonebook_lora8_constant_minus_loss.csv};
\addplot[color=lora16color, mark=square*, thick, solid, opacity=0.7] table[
    col sep=comma,
    x=dataset_size,
    y=best_loss_diff_times_tokens
] {log_parallel_jobs_no_wait_by_taskname/tables_bits/phonebook_lora16_constant_minus_loss.csv};
\addplot[color=lora4color, dashed, thick, forget plot] coordinates {(0, 3906644.05177) (256000, 3906644.05177)};
\addplot[color=lora8color, dashed, thick, forget plot] coordinates {(0, 7813288.10353) (256000, 7813288.10353)};
\addplot[color=lora16color, dashed, thick, forget plot] coordinates {(0, 15626576.20706) (256000, 15626576.20706)};
\end{axis}
\end{tikzpicture}
\caption{\centering bits memorization \\(phonebook)}
\label{fig:phonebook_bits}
\end{subfigure}
\hspace{1cm}
\begin{subfigure}[b]{0.42\textwidth}
\centering
\begin{tikzpicture}
\begin{axis}[
    width=\textwidth,
    height=4.2cm,
    xlabel={\# facts in the training dataset},
    grid=major,
    ymin=0,
    ymax=20000000,
    xmax=190000,
    xlabel style={font=\footnotesize},
    ylabel style={font=\footnotesize},
    title style={font=\footnotesize},
    tick label style={font=\footnotesize},
    xtick={0, 80000, 160000, 240000},
    xticklabels={0, 80k, 160k, 240k},
    ytick={0, 5000000, 10000000, 15000000, 20000000},
    yticklabels={0, 5m, 10m, 15m, 20m},
    scaled ticks = false
]
\addplot[color=lora4color, mark=square*, thick, solid, opacity=0.7] table[
    col sep=comma,
    x=dataset_size,
    y=best_loss_diff_times_tokens
] {log_parallel_jobs_no_wait_by_taskname/tables_bits/arxivpapers_extensive_lora4_constant_minus_loss.csv};
\addplot[color=lora8color, mark=square*, thick, solid, opacity=0.7] table[
    col sep=comma,
    x=dataset_size,
    y=best_loss_diff_times_tokens
] {log_parallel_jobs_no_wait_by_taskname/tables_bits/arxivpapers_extensive_lora8_constant_minus_loss.csv};
\addplot[color=lora16color, mark=square*, thick, solid, opacity=0.7] table[
    col sep=comma,
    x=dataset_size,
    y=best_loss_diff_times_tokens
] {log_parallel_jobs_no_wait_by_taskname/tables_bits/arxivpapers_extensive_lora16_constant_minus_loss.csv};
\addplot[color=lora4color, dashed, thick, forget plot] coordinates {(0, 3906644.05177) (256000, 3906644.05177)};
\addplot[color=lora8color, dashed, thick, forget plot] coordinates {(0, 7813288.10353) (256000, 7813288.10353)};
\addplot[color=lora16color, dashed, thick, forget plot] coordinates {(0, 15626576.20706) (256000, 15626576.20706)};
\end{axis}
\end{tikzpicture}
\caption{\centering bits memorization \\(arXiv)}
\label{fig:arxiv_bits}
\end{subfigure}
\hfill
\begin{subfigure}[b]{0.42\textwidth}
\centering
\begin{tikzpicture}
\begin{axis}[
    width=\textwidth,
    height=4.2cm,
    xlabel={\# facts in the training dataset},
    ylabel={accurate fact count},
    grid=major,
    xmax=280000,
    xlabel style={font=\footnotesize},
    ylabel style={font=\footnotesize},
    title style={font=\footnotesize},
    tick label style={font=\footnotesize},
    xtick={0, 100000, 200000},
    xticklabels={0, 100k, 200k},
    ytick={0, 50000, 100000},
    yticklabels={0, 50k, 100k},
    scaled ticks = false
]
\addplot[color=lora4color, mark=square*, thick, solid, opacity=0.7] table[
    col sep=comma,
    x=dataset_size,
    y=best_perf_times_size
] {log_parallel_jobs_no_wait_by_taskname/tables/phonebook_lora4.csv};
\addplot[color=lora8color, mark=square*, thick, solid, opacity=0.7] table[
    col sep=comma,
    x=dataset_size,
    y=best_perf_times_size
] {log_parallel_jobs_no_wait_by_taskname/tables/phonebook_lora8.csv};
\addplot[color=lora16color, mark=square*, thick, solid, opacity=0.7] table[
    col sep=comma,
    x=dataset_size,
    y=best_perf_times_size
] {log_parallel_jobs_no_wait_by_taskname/tables/phonebook_lora16.csv};
\end{axis}
\end{tikzpicture}
\caption{\centering fact memorization \\(phonebook)}
\label{fig:phonebook_facts}
\end{subfigure}
\hspace{1cm}
\centering
\begin{subfigure}[b]{0.42\textwidth}
\centering
\begin{tikzpicture}
\begin{axis}[
    width=\textwidth,
    height=4.2cm,
    xlabel={\# facts in the training dataset},
    grid=major,
    xlabel style={font=\footnotesize},
    ylabel style={font=\footnotesize},
    title style={font=\footnotesize},
    tick label style={font=\footnotesize},
    xtick={0, 80000, 160000, 240000},
    xticklabels={0, 80k, 160k, 240k},
    ytick={0, 20000, 40000, 60000},
    yticklabels={0, 20k, 40k, 60k},
    scaled ticks = false
]
\addplot[color=lora4color, mark=square*, thick, solid, opacity=0.7] table[
    col sep=comma,
    x=dataset_size,
    y=best_perf_times_size
] {log_parallel_jobs_no_wait_by_taskname/tables/arxivpapers_extensive_lora4.csv};
\addplot[color=lora8color, mark=square*, thick, solid, opacity=0.7] table[
    col sep=comma,
    x=dataset_size,
    y=best_perf_times_size
] {log_parallel_jobs_no_wait_by_taskname/tables/arxivpapers_extensive_lora8.csv};
\addplot[color=lora16color, mark=square*, thick, solid, opacity=0.7] table[
    col sep=comma,
    x=dataset_size,
    y=best_perf_times_size
] {log_parallel_jobs_no_wait_by_taskname/tables/arxivpapers_extensive_lora16.csv};
\end{axis}
\end{tikzpicture}
\caption{\centering fact memorization \\(arXiv)}
\label{fig:arxiv_facts}
\end{subfigure}
\caption{Fact memorization and fact accuracy capacity of LoRA finetuning. See settings in \Cref{sec:selection_experiment}.}
\label{fig:memorization}
\end{figure}
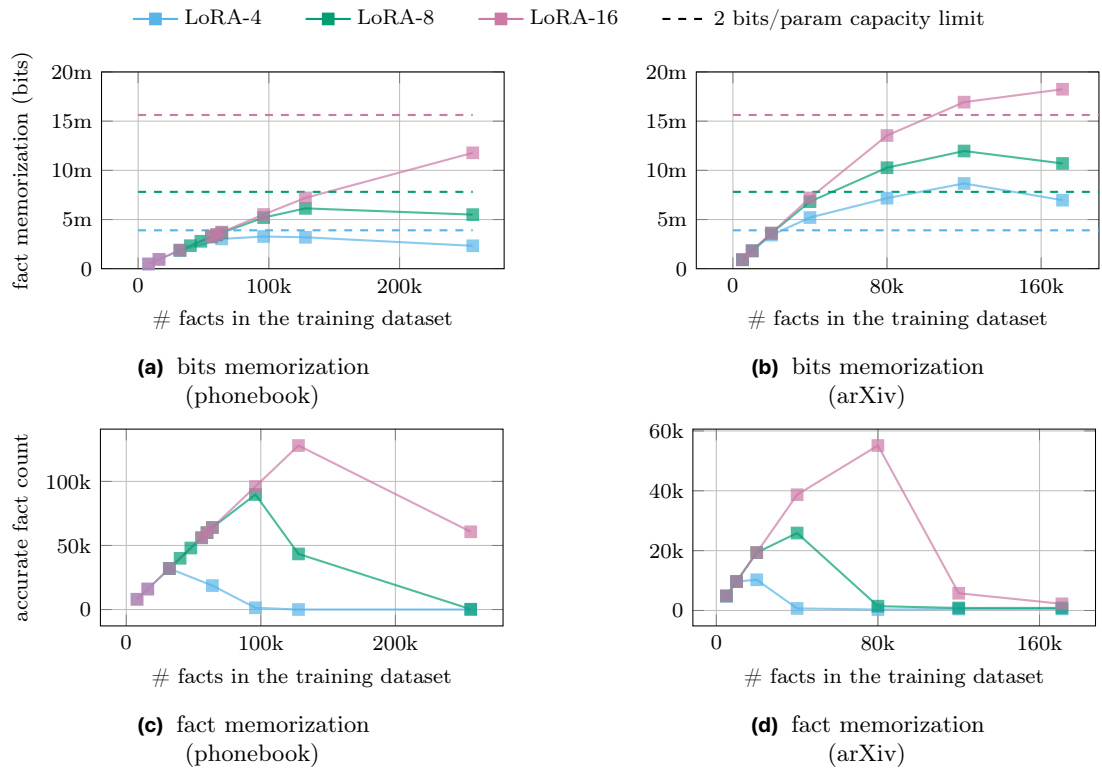

%% file: app_additional_selection_results.tex
\section{Additional Results for Our Selective Training}
\label{app:additional_selection_results}

\subsection{Additional Results for Pretraining}
\label{app:pretrain_additional_selection_results}

\input{tables/tab_phonebook_selection_unwei}
In \Cref{tab:phonebook_selection_unwei}, we show the detailed performance numbers for how our data selection boosts fact accuracy for pretraining on synthetic power law phonebook datasets. Observe that fact memorization can be seen as an unweighted average of prediction accuracy on individual facts, it is natural to ask, does data selection also boost weighted fact accuracy? Indeed in practice, certain facts are more ``important'' to memorize due to their frequent occurrences in the training dataset. In our synthetic phonebook experiments, such weights are naturally the probability of each fact in the underlying power law distributions. Thus in \Cref{tab:phonebook_selection_wei}, we present additional results for weighted fact accuracy performance, and show that it similarly improves under data selection.

\input{tables/tab_phonebook_selection_wei}

\subsection{Additional Results for LoRA-Finetuning: Boosting Fact Accuracy without Additional Forgetting}
\label{app:finetune_additional_selection_results}
We now perform data selection for LoRA finetuning on the semi-synthetic title-author mapping facts in the arXiv-papers dataset~\citep{Saga}.  Besides the finetuning setups mentioned in \Cref{app:suboptimal_fact_mem_LoRA}, we additionally tune the selection ratio $\alpha$ by doing a grid search over $\alpha\in\{0.1, 0.2, \cdots, 1.0\}$ for our selection algorithms. Our results are summarized in  \Cref{fig:forgetting}. We observe that compared to training on the full dataset, our LossHF selection (\Cref{alg:cramless}) significantly improves the fact accuracy (in accurate fact count)  on the arXiv-papers dataset, while achieving similar or better general capability performance (as measured by average accuracy across a set of standard Commonsense~\citep{zellers2019hellaswag,OpenBookQA2018,bisk2020piqa}, MMLU~\citep{hendrycks2020measuring}, and ARC~\citep{allenai:arc} Q\&A tasks following \citep{sanyal2025upweighting}). In \Cref{app:finetune_additional_selection_results}, we further show the detailed performance for each task besides their average. The trend remains roughly the same for individual task performance. This shows that our loss-based selection improves fact accuracy without worsening forgetting during finetuning.
This is consistent with the recent work~\citep{sanyal2025upweighting} that proposes to upweight low-loss samples in the training objective to reduce catastrophic forgetting during finetuning.  However, we remark that the forgetting is still severe in \Cref{fig:forgetting}, and even slightly worsens as LoRA rank increases. It remains an interesting open question to alleviate forgetting during the finetuning that is needed for memorizing new facts.
\input{figures/fig_forgetting}

In \Cref{tab:forgetting}, we further report the detailed performance for various general capabilities tasks. The trend remains roughly the same for each individual task performance, i.e., our loss-based selection improves fact accuracy without worsening forgetting, when compared to LoRA finetuning on the full dataset.

\input{tables/tab_forgetting}

\subsection{Additional Results for Wikipedia Pretraining}
\label{app:wikipedia_additional_results}

\input{fact_count}

\paragraph{Additional Dataset Details and Training Settings} We use the same train, test and validation splits as \citep{zhao2025pre}, and show the number of records and facts in each split in \Cref{tab:fact_counts}, where on average  each record contains around 10 facts.  We pretrain variants of the GPT2-Small (110m parameters), GPT2-medium (335m parameters) and GPT2-large (1.3B parameters) models with context length 1024, sinusoidal positional embeddings, ReLU activation, and post-output layer norm. We follow the hyperparameters in \citep{zhao2025pre} and train for $66000$ steps (around 8 epochs) with  batch-size $320$, using AdamW optimizer with weight decay $0.1$ and cosine learning rate scheduler with fixed learning rate 5e-4,  warmup steps 2000, and gradient clipping norm 5.0. For our selection \Cref{alg:cramless_wiki}, we tune the selection ratio to maximize train fact accuracy by grid search over $\alpha\in\{0.1, 0.2, \cdots, 1.0\}$. For the 110m parameter model,  all results are over three training runs for statistical significance.
\input{lmlmwiki_detailed_plots}

\paragraph{Detailed performance for Wikipedia Pretraining under our Data Selection} Besides the metrics presented in \Cref{ssec:wiki_pretraining_main}, we show in \Cref{fig:lmlmwiki_detail} the detailed performance trend for fact accuracy on training split, general MMLU task, as well as individual NLU tasks over different selection ratios. Observe that fact accuracy on the training split is significantly improved by data selection, similar to test fact accuracy in \Cref{fig:test_fact_accuracy_alpha}; and MMLU performance is similar in trend to Knowledge-MMLU performance in \Cref{fig:MMLU_accuracy_alpha}, despite being more noisy due to the inclusion of more reasoning or comprehension related tasks besides world knowledge; and similar to the trend of average NLU accuracy in \Cref{fig:NLU_accuracy_alpha}, the performances of individual NLU tasks (\Cref{fig:lmlmwiki_detail}(c)-(g)) remain roughly the same across different selection ratio, except for an exceedingly small selection ratio $\alpha=0.1$.

%% file: tables/tab_phonebook_selection_unwei.tex
\begin{table*}[t]
\centering
\scriptsize
\renewcommand{\arraystretch}{1.2}
\setlength{\tabcolsep}{5pt}
\begin{tabular}{l|cc|c c c|cc}
\toprule
\textbf{\# Facts in} & \multicolumn{2}{c|}{\textbf{Full Dataset}} &
\multicolumn{3}{c|}{\textbf{Oracle-Aided Selection}} &
\multicolumn{2}{c}{\textbf{Loss-based Selection}} \\
\textbf{Train Data}&  1x steps & 8x steps  & Flattened & Head & Head-Flattened & LossH (Ours) & LossHF (Ours) \\
\midrule
\multicolumn{8}{l}{\textbf{Power law Exponent $\beta = 0$}} \\
\midrule
1.60M & 1.53M & 1.53M & 1.50M & 1.55M & 1.55M & 1.55M & \textbf{1.60M} \\
1.92M & 1.78M & 1.78M & 1.80M & 1.81M & 1.76M & 1.81M & \textbf{1.91M} \\
2.24M & 1.90M (457k) & 1.90M (457k) & 1.94M & 1.90M & 1.93M (141k) & 1.88M & \textbf{2.05M} \\
2.56M & 0.99M (635k) & 1.85M (136k) & 1.10M (645k) & 1.83M (356k) & 1.93M (525k) & 1.88M & \textbf{2.04M (914k)} \\
5.12M & 0.01M (26k) & 0.01M (5k) & 0.09M & \textbf{2.06M} & 1.95M (471k) & 1.79M & 1.94M \\
10.24M & 0.00M & 0.00M & 0.00M & 1.97M & 1.87M & 1.87M & \textbf{2.11M} \\
\midrule
\multicolumn{8}{l}{\textbf{Power law Exponent $\beta = 0.5$}} \\
\midrule
1.60M & 1.44M & 1.44M & 1.55M & 1.47M & 1.53M & 1.47M & \textbf{1.60M} \\
1.92M & 1.39M & 1.39M & 1.80M & 1.47M & 1.79M & 1.52M & \textbf{1.90M} \\
2.24M & 1.22M (9k) & 1.36M (9k) & 1.98M & 1.52M & 1.76M (706k) & 1.57M & \textbf{2.15M} \\
2.56M & 1.23M (20k) & 1.20M (8k) & 1.44M (678k) & 1.52M (18k) & 1.90M (440k) & 1.54M (20k) & \textbf{2.24M (44k)} \\
5.12M & 0.93M (20k) & 0.93M (20k) & 1.02M & 1.53M & 0.32M (620k) & 1.46M & \textbf{2.16M} \\
10.24M & 0.82M & 0.82M & 0.96M & 1.51M & \textbf{1.64M} & 1.32M & 1.57M \\
\midrule
\multicolumn{8}{l}{\textbf{Power law Exponent $\beta = 1.0$}} \\
\midrule
1.60M & 1.18M & 1.17M & 1.54M & 1.19M & 1.52M & 1.13M & \textbf{1.58M} \\
1.92M & 1.05M & 0.99M & 1.80M & 1.18M & 1.47M & 1.10M & \textbf{1.82M} \\
2.24M & 0.99M (9k) & 0.93M (2k) & \textbf{1.95M} & 1.18M & 1.91M (520k) & 1.06M & 1.80M \\
2.56M & 0.95M (8k) & 0.90M (3k) & 1.40M (645k) & 1.18M (15k) & \textbf{1.83M (489k)} & 1.05M (5k) & 1.74M (23k) \\
5.12M & 0.86M (10k) & 0.85M & 0.10M & 1.19M & \textbf{1.87M (230k)} & 1.00M & 1.72M \\
10.24M & 0.78M & 0.78M & 0.70M & 1.19M & \textbf{1.87M} & 0.97M & 1.27M \\
\midrule
\multicolumn{8}{l}{\textbf{Model Size}: 110M parameters} \\
\multicolumn{8}{l}{\textbf{Accurate Fact Count Capacity Limit}: $(2\text{bits/param}) \times (110\text{M params}) / (22 \times \log_2(10) \text{ bits/fact}) = 3.01\text{M facts}$} \\
\bottomrule
\end{tabular}
\caption{Best {\bfseries fact accuracy} (in accurate fact count, i.e., the expected number of correctly answered facts) for pretraining from scratch on power-law distributed phonebook facts under  different data selection schemes. We bold the best result in each row, and show standard deviation across 10 runs in brackets for certain settings where the performances show more variances (in hindsight being the settings near the capacity threshold, i.e., where the number of model parameters roughly matches the number of facts in the training data). All settings consider sufficiently trained model with 110m parameters that trains for 800k steps, except for the 8x longer training runs which train for 6.4m steps. See \Cref{ssec:sufficient_training_setup} for hyperparameter tuning setups for training on full dataset, and see \Cref{sec:selection_experiment} for hyperparameter tuning setups for training with data selection.}\label{tab:phonebook_selection_unwei}
\end{table*}

%% file: tables/tab_phonebook_selection_wei.tex
\begin{table*}[t]
\centering
\scriptsize
\renewcommand{\arraystretch}{1.2}
\setlength{\tabcolsep}{5pt}
\begin{tabular}{l|cc|c c c|cc}
\toprule
\textbf{\# Facts in} & \multicolumn{2}{c|}{\textbf{Full Dataset}} &
\multicolumn{3}{c|}{\textbf{Oracle-Aided Selection}} &
\multicolumn{2}{c}{\textbf{Loss-based Selection}} \\
\textbf{Train Data}&  1x steps & 8x steps  & Flattened & Head & Head-Flattened & LossH (Ours) & LossHF (Ours) \\
\midrule
\multicolumn{8}{l}{\textbf{Power law Exponent $\beta = 0$}} \\
\midrule
1.60M & 0.942 & 0.942 & 0.936 & 0.962 & 0.966 (0.004) & 0.396 & \textbf{0.999} \\
1.92M & 0.929 & 0.929 & 0.935 & 0.918 & 0.917 (0.115) & 0.931 & \textbf{0.996 (0.002)} \\
2.24M & 0.872 & 0.872 & 0.864 & 0.849 & 0.834 (0.075) & 0.841 & \textbf{0.914 (0.366)} \\
2.56M & 0.677 & 0.726 & 0.739 & 0.735 & 0.739 (0.262) & 0.740 & \textbf{0.798 (0.340)} \\
5.12M & 0.003 & 0.002 & 0.017 & \textbf{0.402} & 0.380 (0.013) & 0.349 & 0.379 (0.117) \\
10.24M & 0.000 & 0.000 & 0.000 & 0.192 & 0.183 (0.015) & 0.183 & \textbf{0.204 (0.102)} \\
\midrule
\multicolumn{8}{l}{\textbf{Power law Exponent $\beta = 0.5$}} \\
\midrule
1.60M & 0.932 & 0.932 & 0.971 & 0.946 & 0.958 (0.034) & 0.944 & \textbf{0.998 (0.000)} \\
1.92M & 0.814 & 0.814 & 0.935 & 0.863 & 0.935 (0.005) & 0.855 & \textbf{0.992 (0.000)} \\
2.24M & 0.692 & 0.727 & 0.885 & 0.796 & 0.861 (0.141) & 0.797 & \textbf{0.971 (0.001)} \\
2.56M & 0.660 & 0.629 & 0.600 & 0.755 & 0.785 (0.247) & 0.724 & \textbf{0.905 (0.011)} \\
5.12M & 0.397 & 0.397 & 0.421 & 0.529 & \textbf{0.569 (0.021)} & 0.460 & 0.566 (0.018) \\
10.24M & 0.276 & 0.276 & 0.289 & 0.373 & \textbf{0.409 (0.007)} & 0.290 & 0.272 (0.092) \\
\midrule
\multicolumn{8}{l}{\textbf{Power law Exponent $\beta = 1.0$}} \\
\midrule
1.60M & 0.971 & 0.972 & 0.969 & 0.973 & 0.953 (0.015) & 0.970 & \textbf{0.998 (0.000)} \\
1.92M & 0.952 & 0.940 & 0.934 & 0.960 & 0.962 (0.007) & 0.955 & \textbf{0.995} \\
2.24M & 0.938 & 0.927 & 0.854 & 0.951 & 0.948 (0.008) & 0.942 & \textbf{0.980 (0.001)} \\
2.56M & 0.928 & 0.916 & 0.684 & 0.943 & 0.904 & 0.933 & \textbf{0.967 (0.001)} \\
5.12M & 0.883 & 0.876 & 0.034 & 0.902 & 0.899 (0.007) & 0.889 & \textbf{0.916} \\
10.24M & 0.840 & 0.840 & 0.333 & \textbf{0.865} & 0.853 (0.027) & 0.850 & 0.858 \\
\midrule
\multicolumn{8}{l}{\textbf{Model Size}: 110M parameters} \\
\bottomrule
\end{tabular}
\caption{Best {\bfseries weighted} fact accuracy (combined weight of accurately answered facts in the underlying data distribution) for pretraining from scratch on power-law distributed phonebook facts under  different data selection schemes. We bold the best result in each row, and show standard deviation across 10 runs in brackets for certain settings where the performances show more variances (in hindsight being the settings near the capacity threshold, i.e., where the number of model parameters roughly matches the number of facts in the training data). All settings consider sufficiently trained model with 110m parameters that trains for 800k steps, except for the 8x longer training runs which train for 6.4m steps. See \Cref{ssec:sufficient_training_setup} for more details on the setup for training on full dataset, and see \Cref{sec:selection_experiment} for more details on the setup for training with data selection.}\label{tab:phonebook_selection_wei}
\end{table*}

%% file: figures/fig_forgetting.tex
\begin{figure}[t]
\centering
\begin{subfigure}[b]{0.42\textwidth}
    \centering
    \begin{tikzpicture}
    \begin{axis}[
        width=\textwidth,
        height=0.75\textwidth,
        xlabel={LoRA rank $r$},
        xmode=log,
        ylabel={accurate fact count},
        legend pos=north west,
        legend style={font=\small},
        grid=major,
        grid style={dashed,gray!30},
        xtick={2,4,8,16,32},
        xticklabels={2,4,8,16,32},
        xmin=3.8, xmax=34,
        ymin=0,
        enlarge y limits=0.1,
        mark size=2.5pt,
        thick,
    ]

    \definecolor{fullcolor}{RGB}{0,114,178}
    \definecolor{cramlesscolor}{RGB}{230,159,0}

    \addplot[
        draw=none,
        forget plot,
        name path=cramless_lower,
    ] table [
        x=lora_r,
        y expr=\thisrow{memcramless}-\thisrow{memcramless_stderr},
        col sep=comma
    ] {tables/tab_forgetting.csv};

    \addplot[
        draw=none,
        forget plot,
        name path=cramless_upper,
    ] table [
        x=lora_r,
        y expr=\thisrow{memcramless}+\thisrow{memcramless_stderr},
        col sep=comma
    ] {tables/tab_forgetting.csv};

    \addplot[cramlesscolor, fill opacity=0.2, forget plot] fill between[of=cramless_lower and cramless_upper];

    \addplot+[
        cramlesscolor,
        mark=*,
        line width=1.2pt,
        mark options={fill=cramlesscolor},
    ] table [
        x=lora_r,
        y=memcramless,
        col sep=comma
    ] {tables/tab_forgetting.csv};
    \addlegendentry{LossHF (Ours)}

    \addplot[
        draw=none,
        forget plot,
        name path=full_lower,
    ] table [
        x=lora_r,
        y expr=\thisrow{memfull}-\thisrow{memfull_stderr},
        col sep=comma
    ] {tables/tab_forgetting.csv};

    \addplot[
        draw=none,
        forget plot,
        name path=full_upper,
    ] table [
        x=lora_r,
        y expr=\thisrow{memfull}+\thisrow{memfull_stderr},
        col sep=comma
    ] {tables/tab_forgetting.csv};

    \addplot[fullcolor, fill opacity=0.2, forget plot] fill between[of=full_lower and full_upper];

    \addplot+[
        fullcolor,
        mark=square*,
        line width=1.2pt,
        mark options={fill=fullcolor},
    ] table [
        x=lora_r,
        y=memfull,
        col sep=comma
    ] {tables/tab_forgetting.csv};
    \addlegendentry{Full-Dataset}

    \addplot[
        black,
        dashed,
        line width=1.2pt,
        domain=2:34,
    ] {0.01};
    \addlegendentry{Pretrained}

    \end{axis}
    \end{tikzpicture}
    \caption{Fact Accuracy}
    \label{fig:forgetting_memorization}
\end{subfigure}
\hspace{1cm}
\begin{subfigure}[b]{0.42\textwidth}
    \centering
    \begin{tikzpicture}
    \begin{axis}[
        width=\textwidth,
        height=0.75\textwidth,
        xlabel={LoRA rank $r$},
        ylabel={average task accuracy},
        xmode=log,
        grid=major,
        grid style={dashed,gray!30},
        xtick={2,4,8,16,32},
        xticklabels={2,4,8,16,32},
        xmin=3.8, xmax=34,
        ymin=0.41, ymax=0.52,
        mark size=2.5pt,
        thick,
    ]

    \definecolor{fullcolor}{RGB}{0,114,178}
    \definecolor{cramlesscolor}{RGB}{230,159,0}

    \addplot[
        draw=none,
        forget plot,
        name path=full_lower_fg,
    ] table [
        x=lora_r,
        y expr=\thisrow{full}-\thisrow{full_stderr},
        col sep=comma
    ] {tables/tab_forgetting.csv};

    \addplot[
        draw=none,
        forget plot,
        name path=full_upper_fg,
    ] table [
        x=lora_r,
        y expr=\thisrow{full}+\thisrow{full_stderr},
        col sep=comma
    ] {tables/tab_forgetting.csv};

    \addplot[fullcolor, fill opacity=0.2, forget plot] fill between[of=full_lower_fg and full_upper_fg];

    \addplot+[
        fullcolor,
        mark=square*,
        line width=1.2pt,
        mark options={fill=fullcolor},
    ] table [
        x=lora_r,
        y=full,
        col sep=comma
    ] {tables/tab_forgetting.csv};

    \addplot[
        draw=none,
        forget plot,
        name path=cramless_lower_fg,
    ] table [
        x=lora_r,
        y expr=\thisrow{cramless}-\thisrow{cramless_stderr},
        col sep=comma
    ] {tables/tab_forgetting.csv};

    \addplot[
        draw=none,
        forget plot,
        name path=cramless_upper_fg,
    ] table [
        x=lora_r,
        y expr=\thisrow{cramless}+\thisrow{cramless_stderr},
        col sep=comma
    ] {tables/tab_forgetting.csv};

    \addplot[cramlesscolor, fill opacity=0.2, forget plot] fill between[of=cramless_lower_fg and cramless_upper_fg];

    \addplot+[
        cramlesscolor,
        mark=*,
        line width=1.2pt,
        mark options={fill=cramlesscolor},
    ] table [
        x=lora_r,
        y=cramless,
        col sep=comma
    ] {tables/tab_forgetting.csv};

    \addplot[
        black,
        dashed,
        line width=1.2pt,
        domain=2:34,
    ] {0.5145};

    \end{axis}
    \end{tikzpicture}
    \caption{General Capabilities Performance}
    \label{fig:forgetting_general}
\end{subfigure}
\caption{Performance comparison between Full-Dataset Training and our LossHF Selection \Cref{alg:cramless} for LoRA finetuning across different ranks on title-authors mapping facts in the arXiv-papers dataset. All  performances are averaged across three runs.  See \Cref{sec:selection_experiment} for more details on the settings, and see \Cref{tab:forgetting} for detailed performances of each general capability task.}
\label{fig:forgetting}
\end{figure}
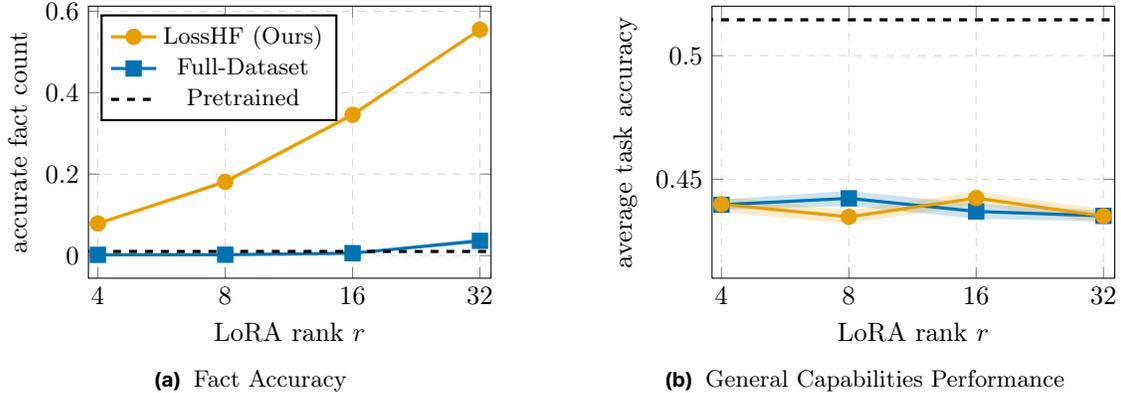

%% file: tables/tab_forgetting.tex
\begin{table*}[t]
\centering
\footnotesize
\renewcommand{\arraystretch}{1.2}
\setlength{\tabcolsep}{5pt}
\begin{tabular}{l|c c c|c |c c|c}
\toprule
\multirow{2}{*}{\textbf{Method}} &
\multicolumn{6}{c|}{\textbf{General Capability Accuracy}} &
\textbf{Target Accuracy} \\
& \multicolumn{3}{c|}{Commonsense}  & MMLU & \multicolumn{2}{c|}{ARC}  & Fact Accuracy (\%) \\
&  Hellaswag & OpenbookQA & PiQA & & Easy & Challenge & on arXiv Papers  \\
\midrule
Pre-trained & 63.1 (0.10) & 36.7 (0.14) & 74.8 (0.09) & 30.7 (0.24) & 65.6 (0.19) & 36.8 (0.10) & 0.1 \\
\midrule
\multicolumn{8}{l}{\textbf{LoRA $r=4$}} \\
\midrule
Full Data  & 54.2 (0.16) & 32.7 (0.79) & 67.8 (0.36) & 24.6 (0.23) & 53.6 (0.82) & 31.0 (0.37) & 0.2 (0.03)\\
\textbf{LossHF (Ours)}  & 55.2 (0.17) & 34.3 (1.35) & 67.7 (0.80) & 26.0 & 50.5 (0.81) & 30.3 (0.65) & 7.9 (0.41) \\
\midrule
\multicolumn{8}{l}{\textbf{LoRA $r=8$}} \\
\midrule
Full Data  & 55.3 (0.39) & 34.2 (1.17) & 65.9 (0.35) & 27.0 (0.78) & 51.6 (0.24) & 31.4 (0.90) & 0.2 (0.06)\\
\textbf{LossHF (Ours)}  & 56.2 (0.43) & 33.7 (0.74) & 67.2 (0.33) & 25.0 (0.86) & 48.3 (0.83) & 30.4 (0.56) & 18.1 (0.57) \\
\midrule
\multicolumn{8}{l}{\textbf{LoRA $r=16$}} \\
\midrule
Full Data  & 56.3 (0.20) & 34.4 (1.25) & 65.5 (0.73) & 25.8 (0.56) & 49.0 (0.81) & 31.2 (0.45) & 0.6 (0.03)\\
\textbf{LossHF (Ours)}  & 56.5 (0.03) & 33.7 (0.98) & 66.7 (0.95) & 26.4 (0.59) & 50.7 (0.17) & 31.5 (0.36) & 34.6 (0.72) \\
\midrule
\multicolumn{8}{l}{\textbf{LoRA $r=32$}} \\
\midrule
Full Data  & 55.9 (0.10) & 32.9 (0.71) & 64.4 (0.25) & 26.8 (0.25) & 49.5 (0.83) & 31.6 (0.67) & 3.6 (0.69)\\
\textbf{LossHF (Ours)}  & 55.9 (0.17) & 33.8 (1.22) & 65.1 (0.36) & 26.9 (0.58) & 48.5 (0.59) & 30.9 (0.75) & 55.5 (0.20) \\
\bottomrule
\end{tabular}
\caption{Our selective training improves fact accuracy without worsening forgetting during LoRA finetuning on title-authors mapping facts in the arXiv-papers dataset. We show mean and standard deviation (in brackets) across three runs.}\label{tab:forgetting}
\end{table*}

%% file: fact_count.tex
\begin{table}
\caption{Fact counts in processed Wikipedia corpus}
\label{tab:fact_counts}
\centering
\begin{tabular}{lrr}
\toprule
Split & Facts & Records \\
\midrule
train & 59670093 & 6245642 \\
validation & 733 & 75 \\
test & 7135 & 830 \\
Total & 59677961 & 6246547 \\
\bottomrule
\end{tabular}
\end{table}

%% file: lmlmwiki_detailed_plots.tex
\pgfplotstableread[col sep=comma]{lmlmwiki_loss_mean_std_110m.csv}\datatableA
\pgfplotstableread[col sep=comma]{lmlmwiki_loss_mean_std_335m.csv}\datatableB
\pgfplotstableread[col sep=comma]{lmlmwiki_loss_mean_std_1_3B.csv}\datatableC

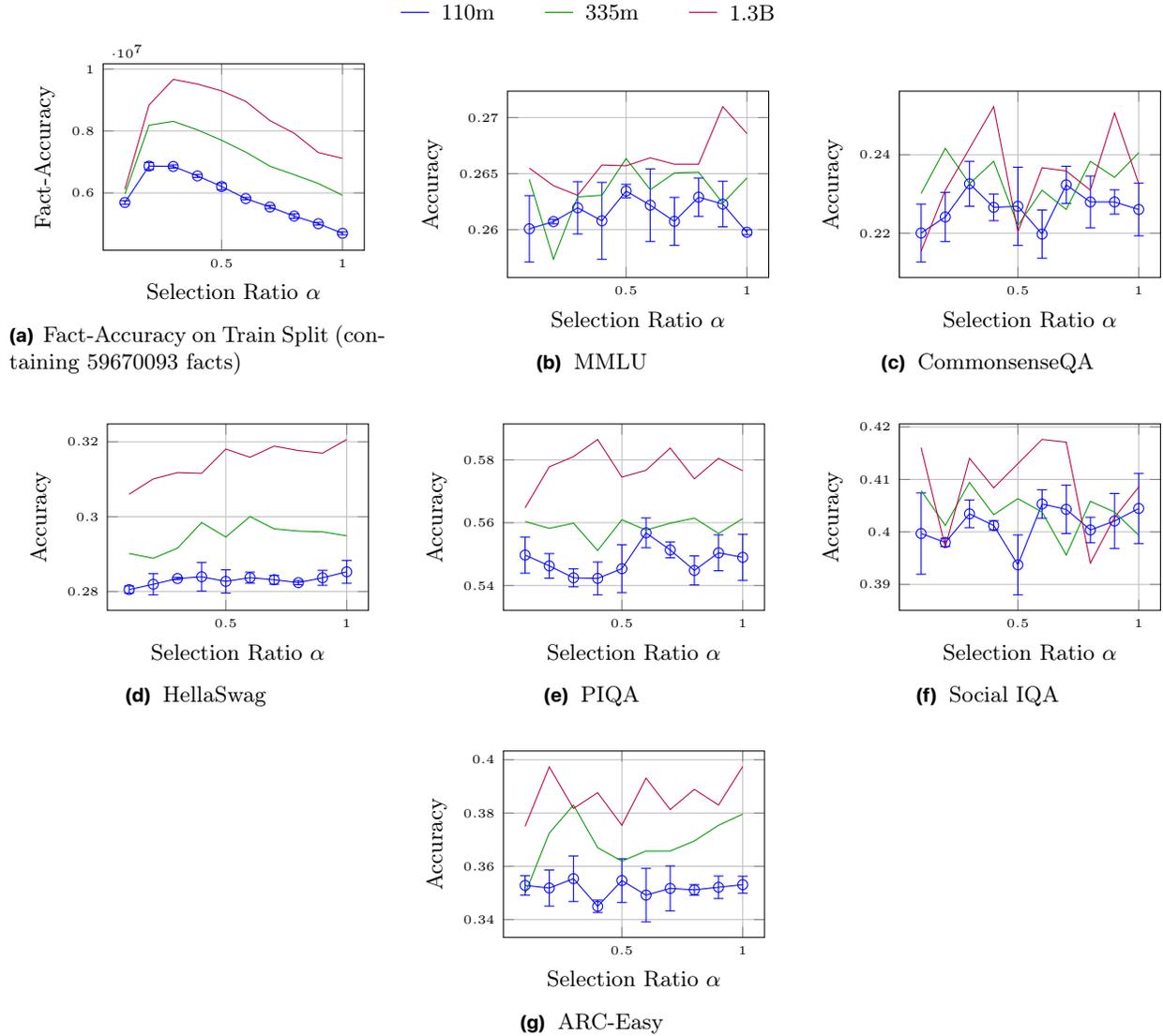
\begin{figure}[htbp]
\centering

\begin{tikzpicture}[baseline=-0.5ex]
\draw[blue] (0,0) -- (0.4,0);
\node[right] at (0.45,0) {\small 110m};
\end{tikzpicture}
\hspace{1em}
\begin{tikzpicture}[baseline=-0.5ex]
\draw[green!60!black] (0,0) -- (0.4,0);
\node[right] at (0.45,0) {\small 335m};
\end{tikzpicture}
\hspace{1em}
\begin{tikzpicture}[baseline=-0.5ex]
\draw[purple] (0,0) -- (0.4,0);
\node[right] at (0.45,0) {\small 1.3B};
\end{tikzpicture}

\vspace{0.2cm}

\begin{subfigure}[b]{0.32\textwidth}
\centering
\begin{tikzpicture}
\begin{axis}[
   width=\textwidth,
   height=0.8\textwidth,
   xlabel={Selection Ratio $\alpha$},
   ylabel={Fact-Accuracy},
   grid=major,
   tick label style={font=\tiny},
   label style={font=\small},
]
\addplot[mark=o, blue, error bars/.cd, y dir=both, y explicit] table[x=fil_rate, y expr=\thisrow{train_perfect_fact_full_acc_mean}*16879403+\thisrow{train_missed_corrected_fact_full_acc_mean}*42790690, y error expr=\thisrow{train_perfect_fact_full_acc_std}*16879403+\thisrow{train_missed_corrected_fact_full_acc_std}*42790690] {\datatableA};
\addplot[no markers, green!60!black] table[x=fil_rate, y expr=\thisrow{train_perfect_fact_full_acc_mean}*16879403+\thisrow{train_missed_corrected_fact_full_acc_mean}*42790690] {\datatableB};
\addplot[no markers, purple] table[x=fil_rate, y expr=\thisrow{train_perfect_fact_full_acc_mean}*16879403+\thisrow{train_missed_corrected_fact_full_acc_mean}*42790690] {\datatableC};
\end{axis}
\end{tikzpicture}
\caption{Fact-Accuracy on Train Split (containing 59670093 facts)}
\end{subfigure}
\hfill
\begin{subfigure}[b]{0.32\textwidth}
\centering
\begin{tikzpicture}
\begin{axis}[
   width=\textwidth,
   height=0.8\textwidth,
   xlabel={Selection Ratio $\alpha$},
   ylabel={Accuracy},
   grid=major,
   tick label style={font=\tiny, /pgf/number format/precision=3, /pgf/number format/fixed},
   label style={font=\small},
]
\addplot[mark=o, blue, error bars/.cd, y dir=both, y explicit] table[x=fil_rate, y=mmlc_0_mean, y error=mmlc_0_std] {\datatableA};
\addplot[no markers, green!60!black] table[x=fil_rate, y=mmlc_0_mean] {\datatableB};
\addplot[no markers, purple] table[x=fil_rate, y=mmlc_0_mean] {\datatableC};
\end{axis}
\end{tikzpicture}
\caption{MMLU}
\end{subfigure}
\hfill
\begin{subfigure}[b]{0.32\textwidth}
\centering
\begin{tikzpicture}
\begin{axis}[
   width=\textwidth,
   height=0.8\textwidth,
   xlabel={Selection Ratio $\alpha$},
   ylabel={Accuracy},
   grid=major,
   tick label style={font=\tiny, /pgf/number format/precision=3, /pgf/number format/fixed},
   label style={font=\small},
]
\addplot[mark=o, blue, error bars/.cd, y dir=both, y explicit] table[x=fil_rate, y=commonsense_qa_mean, y error=commonsense_qa_std] {\datatableA};
\addplot[no markers, green!60!black] table[x=fil_rate, y=commonsense_qa_mean] {\datatableB};
\addplot[no markers, purple] table[x=fil_rate, y=commonsense_qa_mean] {\datatableC};
\end{axis}
\end{tikzpicture}
\caption{CommonsenseQA}
\end{subfigure}

\vspace{0.5cm}

\begin{subfigure}[b]{0.32\textwidth}
\centering
\begin{tikzpicture}
\begin{axis}[
   width=\textwidth,
   height=0.8\textwidth,
   xlabel={Selection Ratio $\alpha$},
   ylabel={Accuracy},
   grid=major,
   tick label style={font=\tiny, /pgf/number format/precision=3, /pgf/number format/fixed},
   label style={font=\small},
]
\addplot[mark=o, blue, error bars/.cd, y dir=both, y explicit] table[x=fil_rate, y=hellaswag_mean, y error=hellaswag_std] {\datatableA};
\addplot[no markers, green!60!black] table[x=fil_rate, y=hellaswag_mean] {\datatableB};
\addplot[no markers, purple] table[x=fil_rate, y=hellaswag_mean] {\datatableC};
\end{axis}
\end{tikzpicture}
\caption{HellaSwag}
\end{subfigure}
\hfill
\begin{subfigure}[b]{0.32\textwidth}
\centering
\begin{tikzpicture}
\begin{axis}[
   width=\textwidth,
   height=0.8\textwidth,
   xlabel={Selection Ratio $\alpha$},
   ylabel={Accuracy},
   grid=major,
   tick label style={font=\tiny, /pgf/number format/precision=3, /pgf/number format/fixed},
   label style={font=\small},
]
\addplot[mark=o, blue, error bars/.cd, y dir=both, y explicit] table[x=fil_rate, y=piqa_mean, y error=piqa_std] {\datatableA};
\addplot[no markers, green!60!black] table[x=fil_rate, y=piqa_mean] {\datatableB};
\addplot[no markers, purple] table[x=fil_rate, y=piqa_mean] {\datatableC};
\end{axis}
\end{tikzpicture}
\caption{PIQA}
\end{subfigure}
\hfill
\begin{subfigure}[b]{0.32\textwidth}
\centering
\begin{tikzpicture}
\begin{axis}[
   width=\textwidth,
   height=0.8\textwidth,
   xlabel={Selection Ratio $\alpha$},
   ylabel={Accuracy},
   grid=major,
   tick label style={font=\tiny, /pgf/number format/precision=3, /pgf/number format/fixed},
   label style={font=\small},
]
\addplot[mark=o, blue, error bars/.cd, y dir=both, y explicit] table[x=fil_rate, y=social_iqa_revised_mean, y error=social_iqa_revised_std] {\datatableA};
\addplot[no markers, green!60!black] table[x=fil_rate, y=social_iqa_revised_mean] {\datatableB};
\addplot[no markers, purple] table[x=fil_rate, y=social_iqa_revised_mean] {\datatableC};
\end{axis}
\end{tikzpicture}
\caption{Social IQA}
\end{subfigure}

\vspace{0.5cm}

\begin{subfigure}[b]{0.32\textwidth}
\centering
\begin{tikzpicture}
\begin{axis}[
   width=\textwidth,
   height=0.8\textwidth,
   xlabel={Selection Ratio $\alpha$},
   ylabel={Accuracy},
   grid=major,
   tick label style={font=\tiny, /pgf/number format/precision=3, /pgf/number format/fixed},
   label style={font=\small},
]
\addplot[mark=o, blue, error bars/.cd, y dir=both, y explicit] table[x=fil_rate, y=arc_easy_mean, y error=arc_easy_std] {\datatableA};
\addplot[no markers, green!60!black] table[x=fil_rate, y=arc_easy_mean] {\datatableB};
\addplot[no markers, purple] table[x=fil_rate, y=arc_easy_mean] {\datatableC};
\end{axis}
\end{tikzpicture}
\caption{ARC-Easy}
\end{subfigure}

\caption{Performance details for Fact-Accuracy on Training split, full MMLU, and individual NLU tasks for using our data selection \cref{alg:cramless_wiki} in pretraining on annotated Wikipedia Corpus (3B tokens) with 66k steps and batch-size 320 (roughly 8 epochs). For the 110m model, we repeat 3 runs and report mean and standard deviation (shown in error bars).}
\label{fig:lmlmwiki_detail}
\end{figure}

%% file: ablations.tex
\section{Ablation Experiments and Discussions}
\label{sec:ablation_and_discussions}

In this section, we perform ablation experiments to understand the inner working of our selection algorithms, as well as discuss design choices and computation cost of our selection algorithms.

\subsection{Comparing to Oracle-Aided Head Selection and Flattening}
\label{sec:ablation_oracle}

To isolate the benefits of the head selection and the flattening step in our data selection schemes, as well as to understand how accurate is loss as a proxy for the underlying fact frequency, in this section, we compare our selection \Cref{alg:cramless} with the following oracle-aided baselines that have precise knowledge of what fact each training record corresponds to, and of the frequency of each fact in the training dataset distribution.

\begin{enumerate}
    \item \textbf{Head}: Select $O\left(\frac{\ln|\mathcal{W}|}{b}\right)$ facts in the training dataset with the highest frequency. This is similar to LossH selection in our \Cref{alg:cramless}, where the loss is replaced by the inverse of  ground-truth frequency of the fact for record $x$.
    \item \textbf{Head-Flattened}: In addition to selection, decrease the sampling probability for facts with high weights to equalize expected non-zero frequencies. This is similar to LossHF selection in \Cref{alg:cramless}, where loss is replaced with  the inverse of  ground-truth frequency of the fact for record $x$.
    \item \textbf{Flattened:} Only reduce the sampling probability for facts with high weights, but do not throw away any tail facts with low weights. This is an baseline algorithm to understand the effect of flattening without head selection, and is similar to LossHF selection in \Cref{alg:cramless}, where loss is replaced with the inverse of ground-truth frequency of the fact for record $x$, and the sampling probability for low-frequency facts is set to one instead of zero.
\end{enumerate}

Our results are summarized in \Cref{tab:phonebook_selection_unwei}, where for most entries, we show the optimal run after tuning training hyperparameters following \Cref{sec:cap_speed_mem_standard_training} and tuning the selection ratio via grid search over $\alpha\in\{0.1, 0.2, \cdots, 1.0\}$; for settings where the model and training dataset are at capacity of each other, we observe high variance of results across repeated runs with the same hyperparameters (intuitively due to an edge-of-stability phenomenon near capacity threshold), and thus report the median (standard deviation) across 10 repeated runs. Our first observation is that among the three oracle-aided methods, Head-Flattened consistently performs the best, significantly outperforming other oracle-aided baselines including Flattened and Head. The comparison between Flattened and Head shows some nuances: Flattened outperforms Head when the model is sufficiently large to fit all facts in the training dataset (facts$\leq$2560000); by contrast Head outperforms Flattened when the number of facts in the training dataset exceeds model fitting power (facts$\geq$2560000). This validates that both the flattening step (which down-samples facts that appear too frequently) and the head selection step (which removes rare facts that exceed the model's fitting power) are necessary for boosting fact memorization to the capacity limit.

Our second observation is that our LossHF and LossH selection (\Cref{alg:cramless}), despite solely using loss for selection (and not having any prior knowledge on the fact frequency or entropy), consistently perform on par with the oracle-aided Head-Flattened selection and Head selection respectively. LossHF (LossH) are only worse than Head-Flattened (Head) when the training dataset contains an extremely large number of facts that are distributed as a power law with high exponent ($\beta=1.0$). This validates the \textbf{\textit{effectiveness of using loss alone}} to approximately distinguish rare facts versus redundant facts, as also illustrated in the data usage histograms \Cref{fig:hist_data_usage} for various selection methods.  In \Cref{ssec:loss_selection_score}, we  discuss additional nuances in the design of loss-based selection score.

\subsection{On Design Choices of Loss-based Selection Score}
\label{ssec:loss_selection_score}
There are many  confounding factors that may affect the quality of loss-based approximation for fact weight, such as the training time, sequence length, and the hardness of each training data sequence. Below we discuss important design choices in our loss-based selection score to adapt to these factors. 

\paragraph{Why not select by token-level loss?} Many data selection methods for language model training operate at the token level, including the celebrated Rho-1~\citep{lin2024not} method. In this paper, we choose to perform selection based on per-record loss (\Cref{alg:cramless}) or per-fact loss (\Cref{alg:cramless_wiki}) rather than per-token loss, to preserve the boundary of each fact and ensure the memorization of the whole fact. Indeed, in validation experiments we observe that  under the same settings of \Cref{tab:forgetting}, a token-level selection variant of \Cref{alg:cramless} would ignore fact boundaries, and ultimately still results in zero fact memorization no matter what selection ratio is used, i.e., yielding suboptimal fact accuracy similar to  training on full-dataset. 

\input{figures/fig_cali_speed}
\paragraph{Why do we need an online threshold?} This is intuitively to adapt to the training dynamics. Indeed, in experiments (\Cref{fig:cali_speed}) we observe that the Spearman rank correlation between negative per-sequence loss and fact-weight generally improve as training proceeds, especially when the training data is above the capacity of the model (i.e., when the fact accuracy is low) which is precisely the regime that we are interested in improving in this paper.

\paragraph{How to Calibrate Loss to Sequence Difficulty?} To control difficulty levels of training data sequences, we further vary the prefix length and suffix length of our synthetic phonebook dataset. We construct a heterogeneous phonebook dataset that consists of $16$ equal-sized groups of different difficulty levels, specified by different prefix length (among $\{6, 9, 12, 15\}$) and suffix length (among $\{12, 18, 24, 30\}$) of phonebook records. The records in each group follows a power law distribution with exponent $1.0$. Interestingly, in this setting with manually introduced heterogeneous facts, we observe in \Cref{fig:cal_seq_len} that the \textit{sum of per-token loss over a sequence} incurs significantly stronger rank correlation to the fact weight-to-bits ratio, when compared to the \textit{average of per-token loss over a sequence}. This benefit of using sum rather than average of per-token loss is consistent with the form of sum-of-loss-based memorization lower bounds in our~\Cref{cor:mem_lower_loss} as well as in  prior works \citet[Theorem 3.2]{allen2024physics} and \citet[Section 2.3]{morris2025much}.

\input{figures/fig_seq_len}

\subsection{On the Computational Cost of Our Selection Algorithms}

\paragraph{Increased training cost of selection due to batch accumulation} In \Cref{alg:cramless}, the batch accumulation step keeps the number of backward passes in each iteration constant (proportional to $b$), but increase the number of forward passes per iteration (as one may need to compute loss on multiple batches to select $b$ records), thus increasing the required training FLOPs per iteration. Nevertheless, we strive to make a fair comparison between training with and without data selection, by focusing on the setting of sufficient training. In such regimes, the performance gain of increasing FLOPs via longer training (8x training steps) on the full dataset is negligible, as shown by \Cref{fig:detailed_gap_power} and \Cref{tab:phonebook_selection_unwei}. We leave it as an interesting open problem as to comparing training with and without data selection in the bounded training FLOPs regimes.

\paragraph{Cost of Tuning the Selection Ratio $\alpha$} Another potentially computationally expensive component of training with data selection, is to determine $\alpha$, i.e., the fraction of dataset to keep. In all experiments of the paper, we perform grid search for the optimal $\alpha$, as we focus in the sufficient training FLOPs setting. However, to understand the possibility of more \textit{efficiently} choosing $\alpha$ in practice, we further investigate the below two questions.

\begin{enumerate}
    \item Is it possible to \textbf{\textit{build data selection scaling laws}} to \textbf{\textit{predict optimal $\alpha$ for large models from tuned $\alpha$ for small models}}? Suppose that the underlying task rely on power-law distributed facts $\text{fact}_1, \cdots, \text{fact}_N$ with power law exponent $\beta$, i.e., $\Pr\left[\text{fact}_i\right] \propto \frac{1}{i^\beta}$ and suppose that each fact on average contains $b$ bits of information. Then the maximal fraction of training dataset that can be memorized on model in discrete space $\mathcal{W}$ (with proportional to $\ln|\mathcal{W}|$ parameters) is as follows. 
    \begin{align}
        \alpha(\mathcal{W})\propto\sum_{i=1}^{\ln |\mathcal{W}| /b}\frac{1}{i^\beta} \propto \begin{cases}
            1 -  \text{constant}\cdot \frac{1}{ \left(\ln|\mathcal{W}| \right) ^{\beta-1}}  & \beta>1\\
            \ln  \ln |\mathcal{W}| - \text{constant}  & \beta = 1\\
             \text{constant}\cdot \ln|\mathcal{W}| ^{1-\beta} - 1& 0<\beta<1\\
            \ln|\mathcal{W}| & \beta = 0
        \end{cases}
    \end{align}
    
    \input{figures/selection_lora}
    Empirically, we observe in \Cref{fig:alpha_model_size} that for LoRA finetuning with different ranks on the real-world arxiv-papers dataset, such a scaling law with zero power law exponent (i.e., uniform distribution) emerges, i.e., as model size increases, the optimal $\alpha$ increases proportionally. This suggests that arXiv papers' title-to-authors mapping is close to uniformly distributed.
    \item Can we  \textbf{\textit{predict optimal $\alpha$ for final trained model from tuned $\alpha$ on earlier training checkpoints}?} This potentially allows one to perform faster search of $\alpha$ by training for a fewer number of steps. Intuitively, one may expect this approach to be feasible due to the hypothetical monotonic relationship between fact memorization at earlier training steps versus the fact memorization of final trained model. Perhaps surprisingly, in \Cref{fig:alpha_training_time}, we partially refute this hypothesis by observing that the optimal alpha first increases and then stabilizes as training proceeds. This suggests that the optimal selection ratio changes between the infinite training FLOPs setting (more training steps) and the bounded training FLOPs setting (earlier checkpoints). We leave optimal data selection under bounded training FLOPs as interesting open problem.
\end{enumerate}

%% file: figures/fig_cali_speed.tex
\setlength{\intextsep}{1pt}
\begin{figure}
    \centering
    \resizebox{0.37\textwidth}{!}{
    \begin{tikzpicture}

    \begin{axis}[
        legend cell align={left},
        legend style={
          fill opacity=0.8,
          draw opacity=1,
          text opacity=1,
          at={(0.82,0.2)},
          anchor=south east,
          draw=white!80!black
        },
        tick align=outside,
        tick pos=left,
        x grid style={white!70!black},
        xlabel={\# tokens seen / total tokens in all facts},
        xmin=0, xmax=800,
        xtick={0, 100, 200, 300, 400, 500, 600, 700, 800},
        xticklabels={0, 100, 200, 300, 400, 500, 600, 700, 800},
        xtick style={color=black},
        scaled ticks=false,
        y grid style={white!70!black},
        ylabel={},
        ymin=0, ymax=1,
        ytick={0.25, 0.5, 0.75,1.0},
        yticklabels={0.25, 0.5, 0.75,1.0},
        ytick style={color=black},
        scaled y ticks=false,
        ylabel near ticks,
        ylabel shift={0pt}
    ]

    \addplot [green!50!black, mark=*, mark size=1.5, mark options={solid,fill=green!50!black}] table[col sep=comma,x=epoch,y expr=-\thisrow{autoloss_vs_weights}] {evals/mem_speed/L12D768H12_power0_5_sample640000_vqnwd56zfg/aaa_spearmanr.csv};
    \addlegendentry{640k facts (correlation)}

    \addplot [blue, mark=*, mark size=1.5, mark options={solid,fill=blue}] table[col sep=comma,x=epoch,y expr=-\thisrow{autoloss_vs_weights}] {evals/mem_speed/L12D768H12_power0_5_sample2560000_6fne4nfrmx/aaa_spearmanr.csv};
    \addlegendentry{2.56m facts (correlation)}

    \addplot [orange, mark=*, mark size=1.5, mark options={solid,fill=orange}] table[col sep=comma,x=epoch,y expr=-\thisrow{autoloss_vs_weights}] {evals/mem_speed/L12D768H12_power0_5_sample10240000_cau4bms2am/aaa_spearmanr.csv};
    \addlegendentry{10.24m facts (correlation)}

    \addplot [green!50!black, dashed, mark=-, mark size=1.5, mark options={solid,fill=green!50!black}] table[col sep=comma,x=epoch,y=valid_prefixed_full_acc] {evals/mem_speed/L12D768H12_power0_5_sample640000_vqnwd56zfg/aaa_spearmanr.csv};
    \addlegendentry{640k facts (accuracy)}

    \addplot [blue, dashed, mark=-, mark size=1.5, mark options={solid,fill=blue}] table[col sep=comma,x=epoch,y=valid_prefixed_full_acc] {evals/mem_speed/L12D768H12_power0_5_sample2560000_6fne4nfrmx/aaa_spearmanr.csv};
    \addlegendentry{2.56m facts (accuracy)}

    \addplot [orange, dashed, mark=-, mark size=1.5, mark options={solid,fill=orange}] table[col sep=comma,x=epoch,y=valid_prefixed_full_acc] {evals/mem_speed/L12D768H12_power0_5_sample10240000_cau4bms2am/aaa_spearmanr.csv};
    \addlegendentry{10.24m facts (accuracy)}

    \end{axis}

    \end{tikzpicture}
    }
    \caption{Dynamics of Spearman rank correlation (between negative per-sequence loss and fact weight) and fact accuracy in training a 110m model on synthetic power law   (with exponent $0.5$)  distributed phonebook datasets containing different number of facts. See detailed settings in \Cref{ssec:sufficient_training_setup}. }
    \label{fig:cali_speed}
\end{figure}
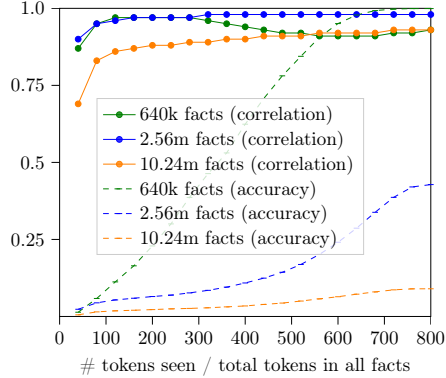

%% file: figures/fig_seq_len.tex
\begin{figure*}
    \centering
    \begin{tabular}{cc}
       \includegraphics[width=0.4\textwidth]{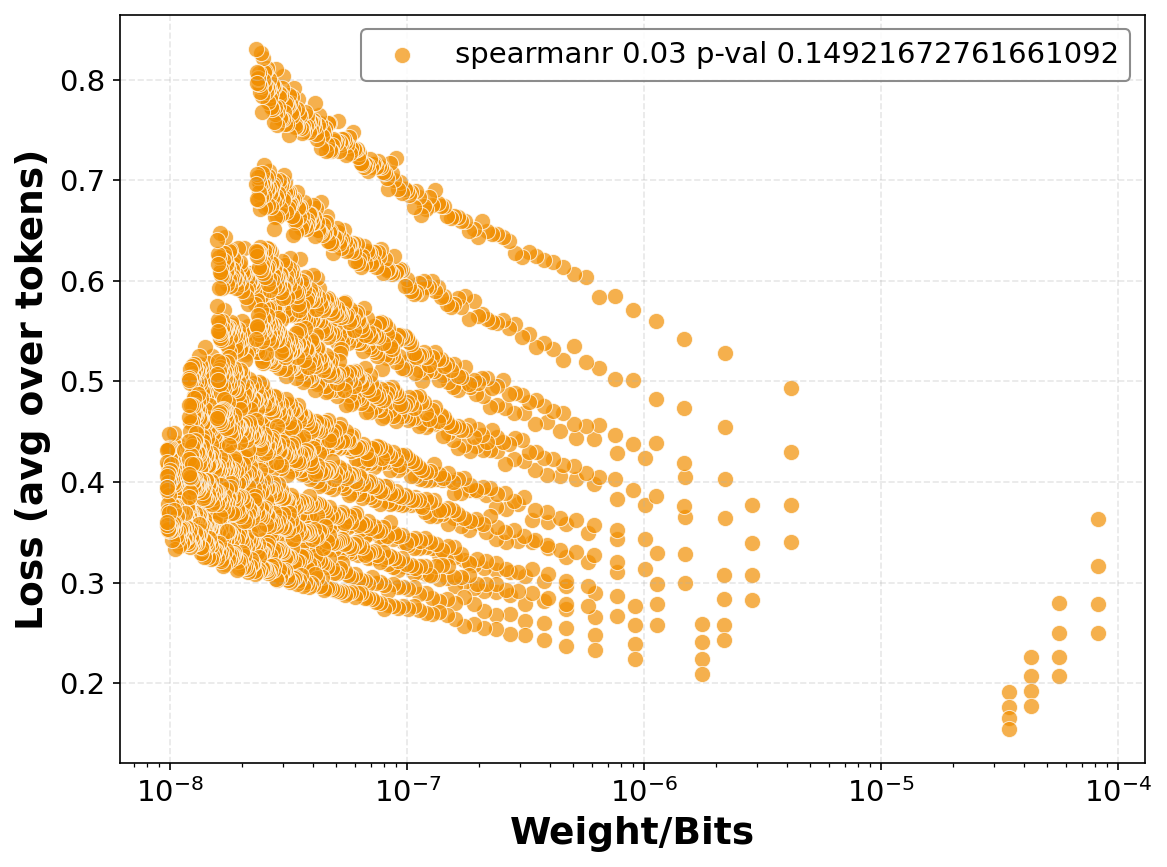}  & \includegraphics[width=0.4\textwidth]{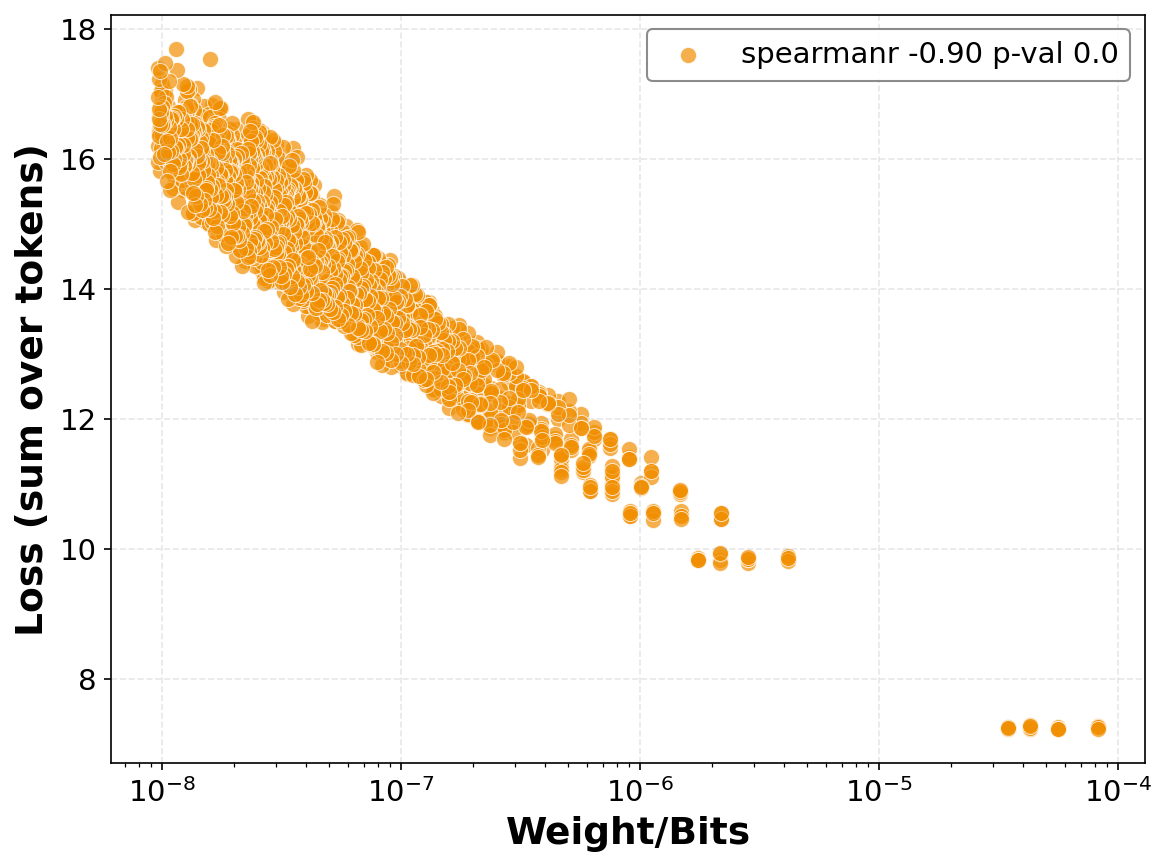}  
    \end{tabular}
    \caption{Sum of per-token loss over a sequence (right plot) shows stronger rank correlation to the fact-weight-to-bits, when compared to average of per-token loss over a sequence (left plot). We consider the setting of training on mixture of power law phonebook datasets with heterogeneous prefix lengths and suffix lengths. See detailed settings in \Cref{ssec:loss_selection_score}.}
    \label{fig:cal_seq_len}
\end{figure*}

%% file: figures/selection_lora.tex
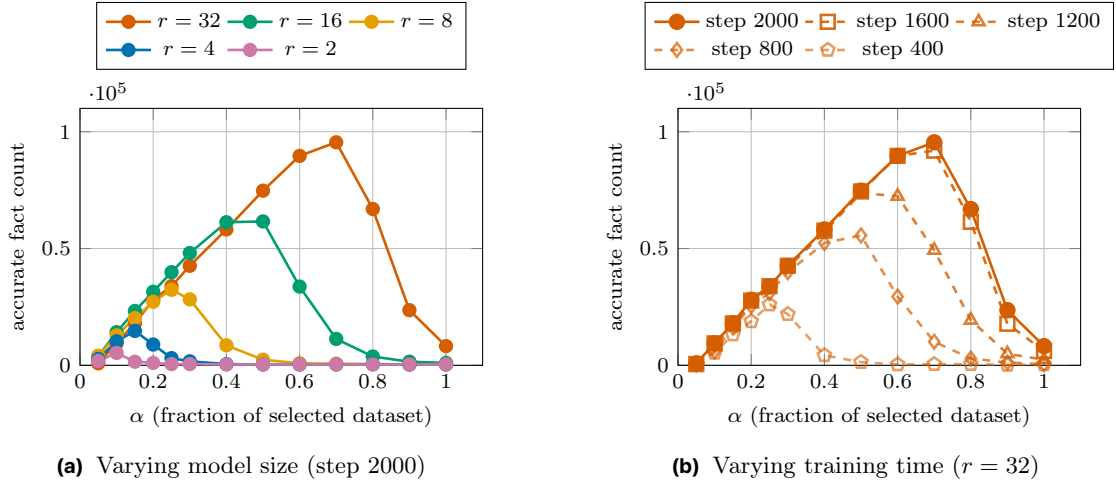
\begin{figure}[htbp]
  \centering

  \begin{subfigure}[b]{0.42\textwidth}
    \centering
    \begin{tikzpicture}
      \begin{axis}[
        width=\textwidth,
        height=5cm,
        xlabel={$\alpha$ (fraction of selected dataset)},
        ylabel={accurate fact count},
        legend style={at={(0.5,1.15)}, anchor=south, legend columns=3,
        font=\footnotesize},
        grid=major,
        xmin=0, xmax=1.1,
        ymin=0, ymax=110000,
        xlabel style={font=\footnotesize},
        ylabel style={font=\footnotesize},
        title style={font=\footnotesize},
        tick label style={font=\footnotesize},
        mark size=2.2pt,
        mark options={solid},
        shader=flat,
      ] 
        
        \addplot[
          color={rgb,255:red,213;green,94;blue,0},
          mark=*,
          mark options={solid, fill={rgb,255:red,213;green,94;blue,0}},
          line width=1pt,
        ] table[col sep=comma, x=alpha, y=best_perf_times_size]
          {log_parallel_jobs_no_wait_by_taskname/tables_alpha/arxivpapers_extensive_seqrho1_lora32_171104.csv};

        \addplot[
          color={rgb,255:red,0;green,158;blue,115},
          mark=*,
          mark options={solid, fill={rgb,255:red,0;green,158;blue,115}},
          line width=1pt,
        ] table[col sep=comma, x=alpha, y=best_perf_times_size]
          {log_parallel_jobs_no_wait_by_taskname/tables_alpha/arxivpapers_extensive_seqrho1_lora16_171104.csv};

        \addplot[
          color={rgb,255:red,230;green,159;blue,0},
          mark=*,
          mark options={solid, fill={rgb,255:red,230;green,159;blue,0}},
          line width=1pt,
        ] table[col sep=comma, x=alpha, y=best_perf_times_size]
          {log_parallel_jobs_no_wait_by_taskname/tables_alpha/arxivpapers_extensive_seqrho1_lora8_171104.csv};

        \addplot[
          color={rgb,255:red,0;green,114;blue,178},
          mark=*,
          mark options={solid, fill={rgb,255:red,0;green,114;blue,178}},
          line width=1pt,
        ] table[col sep=comma, x=alpha, y=best_perf_times_size]
          {log_parallel_jobs_no_wait_by_taskname/tables_alpha/arxivpapers_extensive_seqrho1_lora4_171104.csv};

        \addplot[
          color={rgb,255:red,204;green,121;blue,167},
          mark=*,
          mark options={solid, fill={rgb,255:red,204;green,121;blue,167}},
          line width=1pt,
        ] table[col sep=comma, x=alpha, y=best_perf_times_size]
          {log_parallel_jobs_no_wait_by_taskname/tables_alpha/arxivpapers_extensive_seqrho1_lora2_171104.csv};

        \legend{$r=32$, $r=16$, $r=8$, $r=4$, $r=2$}
      \end{axis}
    \end{tikzpicture}
    \subcaption{Varying model size (step 2000)}\label{fig:alpha_model_size}
  \end{subfigure}
  \hspace{1cm}
  \begin{subfigure}[b]{0.42\textwidth}
    \centering
    \begin{tikzpicture}
      \begin{axis}[
        width=\textwidth,
        height=5cm,
        xlabel={$\alpha$ (fraction of selected dataset)},
        ylabel={accurate fact count},
        legend style={at={(0.5,1.15)}, anchor=south, legend columns=3,
        font=\footnotesize},
        grid=major,
        xmin=0, xmax=1.1,
        ymin=0, ymax=110000,
        xlabel style={font=\footnotesize},
        ylabel style={font=\footnotesize},
        title style={font=\footnotesize},
        tick label style={font=\footnotesize},
        mark size=2.6pt,
        mark options={solid},
        shader=flat,
      ]

        \addplot[
          color={rgb,255:red,213;green,94;blue,0},
          mark=*,
          mark options={solid, fill={rgb,255:red,213;green,94;blue,0}},
          line width=1pt,
          opacity=1,
        ] table[col sep=comma, x=alpha, y expr=\thisrow{performance} * 171104]
          {log_parallel_jobs_arxiv_extensive_1k_early_ckpts/qjwfzzbcuh_ckpt2000.csv};

        \addplot[
          color={rgb,255:red,213;green,94;blue,0},
          mark=square,
          mark options={solid, fill={rgb,255:red,213;green,94;blue,0}},
          line width=1pt,
          dashed,
          opacity=0.9,
        ] table[col sep=comma, x=alpha, y expr=\thisrow{performance} * 171104]
          {log_parallel_jobs_arxiv_extensive_1k_early_ckpts/qjwfzzbcuh_ckpt1600.csv};

        \addplot[
          color={rgb,255:red,213;green,94;blue,0},
          mark=triangle,
          mark options={solid, fill={rgb,255:red,213;green,94;blue,0}},
          line width=1pt,
          dashed,
          opacity=0.8,
        ] table[col sep=comma, x=alpha, y expr=\thisrow{performance} * 171104]
          {log_parallel_jobs_arxiv_extensive_1k_early_ckpts/qjwfzzbcuh_ckpt1200.csv};

        \addplot[
          color={rgb,255:red,213;green,94;blue,0},
          mark=diamond,
          mark options={solid, fill={rgb,255:red,213;green,94;blue,0}},
          line width=1pt,
          dashed,
          opacity=0.7,
        ] table[col sep=comma, x=alpha, y expr=\thisrow{performance} * 171104]
          {log_parallel_jobs_arxiv_extensive_1k_early_ckpts/qjwfzzbcuh_ckpt800.csv};

        \addplot[
          color={rgb,255:red,213;green,94;blue,0},
          mark=pentagon,
          mark options={solid, fill={rgb,255:red,213;green,94;blue,0}},
          line width=1pt,
          dashed,
          opacity=0.6,
        ] table[col sep=comma, x=alpha, y expr=\thisrow{performance} * 171104]
          {log_parallel_jobs_arxiv_extensive_1k_early_ckpts/qjwfzzbcuh_ckpt400.csv};

        \legend{step 2000, step 1600, step 1200, step 800, step 400}
      \end{axis}
    \end{tikzpicture}
    \subcaption{Varying training time ($r=32$)}\label{fig:alpha_training_time}
  \end{subfigure}

  \caption{Fact accuracy (in accurate fact count) versus dataset selection ratio $\alpha$ for LoRA
    finetuned model on arXiv-papers dataset. We vary the LoRA rank $r$ (left plot) to control model size, and vary the number of training steps (right plot).}
  \label{fig:memorization_alpha}
  \vspace{0.3cm}
\end{figure}